\documentclass[referee,sn-aps]{sn-jnl}
\usepackage{nicefrac}
\usepackage{graphicx}%
\usepackage{multirow}%
\usepackage{makecell}
\usepackage{amsmath,amssymb,amsfonts}%
\usepackage{amsthm}%
\usepackage{mathrsfs}%
\usepackage[title]{appendix}%
\usepackage{xcolor}%
\usepackage{textcomp}%
\usepackage{manyfoot}%
\usepackage{booktabs}%
\usepackage{algorithm}%
\usepackage{algorithmicx}%
\usepackage{nicefrac}
\usepackage{algpseudocode}%
\usepackage{listings}%
\usepackage{booktabs}
\usepackage{graphicx}
\usepackage{lineno}
\usepackage{subcaption}
\usepackage[frozencache,cachedir=minted-cache]{minted}
\usepackage{xcolor}
\definecolor{codegray}{HTML}{f5f5f5}
\setminted[python]{    
    linenos=true,
    fontsize=\scriptsize,
    frame=leftline,
    framesep=2mm,
    bgcolor=codegray,
    breaklines,
    breakanywhere,
    baselinestretch=1
}
\usepackage{tabularx}

\newtheorem{theorem}{Theorem}%

\newtheorem{example}{Example}%
\newtheorem{corollary}{Corollary}

\begin{document}

\title{PINNIES: An Efficient Physics-Informed Neural Network Framework to Integral Operator Problems}

\author*[1]{\fnm{Alireza} \sur{Afzal Aghaei}}\email{alirezaafzalaghaei@gmail.com}

\author[2]{\fnm{Mahdi} \sur{Movahedian Moghaddam}}\email{m\_movahedian@sbu.ac.ir}

\author[3]{\fnm{Kourosh} \sur{Parand}}\email{k\_parand@sbu.ac.ir}

{
\affil[1]{\small Independent Researcher, \state{Isfahan}, \country{Iran}}

\affil[2]{\small\orgdiv{Department of Computer and Data Sciences}, \orgname{Shahid Beheshti University}, \orgaddress{ \city{Tehran}, \state{Tehran}, \country{Iran}}}

\affil[3]{\small\orgdiv{Department of Statistics and Actuarial Science}, \orgname{University of Waterloo}, \orgaddress{ \city{Waterloo}, \state{Ontario}, \country{Canada}}}
}
\abstract{
This paper introduces an efficient tensor-vector product technique for the rapid and accurate approximation of integral operators within physics-informed deep learning frameworks. Our approach leverages neural network architectures to evaluate problem dynamics at specific points, while employing Gaussian quadrature formulas to approximate the integral components, even in the presence of infinite domains or singularities. We demonstrate the applicability of this method to both Fredholm and Volterra integral operators, as well as to optimal control problems involving continuous time. Additionally, we outline how this approach can be extended to approximate fractional derivatives and integrals and propose a fast matrix-vector product algorithm for efficiently computing the fractional Caputo derivative.
In the numerical section, we conduct comprehensive experiments on forward and inverse problems. For forward problems, we evaluate the performance of our method on over 50 diverse mathematical problems, including multi-dimensional integral equations, systems of integral equations, partial and fractional integro-differential equations, and various optimal control problems in delay, fractional, multi-dimensional, and nonlinear configurations. For inverse problems, we test our approach on several integral equations and fractional integro-differential problems. Finally, we introduce the \texttt{pinnies} Python package to facilitate the implementation and usability of the proposed method.
}



\maketitle

\section{Introduction}\label{sec1}

Mathematical problems have been a cornerstone of scientific and technological advancements, offering a framework to model, analyze, and solve a wide range of real-world challenges \cite{mainardi2022fractional}. From the intricacies of quantum mechanics to the complexities of economic systems, mathematical formulations provide a precise and systematic approach to understanding and predicting phenomena. The applications of mathematical problems are diverse and far-reaching, encompassing fields such as physics, engineering, biology, and finance \cite{meerschaert2013mathematical, polyanin2008handbook,aghaei2023solving}. 

A significant subset of mathematical problems is defined by integral operators, which play a crucial role in various disciplines \cite{polyanin2008handbook,wazwaz2011linear}. Integral equations (IEs), for example, are used to model problems where the unknown function appears under an integral sign. These equations are particularly valuable in fields such as electromagnetism, fluid dynamics, and quantum mechanics, where they aid in describing scattering phenomena and situations where the rate of change of a quantity depends on both its current state and its history \cite{wazwaz2011linear, colton2013integral}. A more illustrative example is found in fields like viscoelasticity, where the stress in a material depends on the strain history, and in control theory, where systems with memory effects are modeled \cite{mainardi2022fractional, podlubny1998fractional}. Optimal control problems, another class of problems involving integral operators, are essential in control theory and optimization. Here, the goal is to find the best control strategy that minimizes or maximizes a given performance criterion, often formulated as an integral functional. Such problems are critical in engineering, economics, and operations research, where they are applied to optimize processes, manage resources, and make strategic decisions \cite{lewis2012optimal}.

Forward and inverse problems represent another critical area of mathematical inquiry, particularly in the context of applied sciences and engineering. Forward problems involve determining the outcome or behavior of a system based on a known set of parameters or initial conditions. These problems are typically well-posed, meaning they have a unique solution that depends continuously on the input data. For example, in heat conduction, given the initial temperature distribution and boundary conditions, the forward problem seeks to predict the temperature at any future time. In contrast, inverse problems are concerned with determining the unknown causes or parameters from observed data. These are often ill-posed, as small changes in the data can lead to large variations in the solution, making them more challenging to solve. Inverse problems are ubiquitous in areas such as medical imaging, geophysics, and machine learning, where they are used to reconstruct images from measurements, infer the structure of the Earth’s interior from seismic data, or deduce model parameters from observed outcomes \cite{ito2014inverse}.

Recently, physics-informed deep learning models have emerged as powerful tools for solving complex mathematical equations, particularly differential equations that are prevalent in scientific and engineering problems in forward and inverse forms. These approaches leverage the strengths of neural networks to approximate solutions to differential equations by incorporating physical laws and constraints directly into the learning process. Unlike traditional numerical methods, which often require significant computational resources and can struggle with high-dimensional problems, physics-informed neural networks (PINNs) offer a flexible and scalable alternative \cite{RAISSI2019686}. By embedding the governing equations of a system, such as conservation laws, fluid dynamics, or quantum mechanics, into the loss function of a neural network, PINNs ensure that the learned solution adheres to the underlying physical principles. This integration not only enhances the accuracy and reliability of the models but also enables the solution of forward, inverse, and hybrid problems that are otherwise intractable using conventional methods. As a result, these models have been successfully applied to a wide range of problems, from fluid flow \cite{cai2021physics} and heat transfer \cite{cai2021physicsheat} to electromagnetic \cite{zhang2024solution, zhang2022physics} wave propagation, and material science, demonstrating their potential to revolutionize the way mathematical equations are solved in various domains.

These advanced approaches are made possible by the implementation of automatic differentiation, a technique that allows the computation of derivatives with respect to all input variables of a neural network efficiently and accurately \cite{baydin2018automatic}. Automatic differentiation, integrated within modern deep learning frameworks, provides a seamless way to calculate gradients, which are essential for training neural networks and optimizing loss functions that include differential equations. This capability is particularly crucial for PINNs, where the loss function often involves derivatives of the neural network output with respect to its inputs, representing the differential operators in the governing equations. This method not only reduces the complexity of implementing traditional differentiation techniques but also enhances the accuracy, performance, and scalability of deep learning models, enabling the solution of highly complex and nonlinear differential equations that arise in real-world applications.

In contrast to automatic differentiation, which computes the exact, analytical derivative of a function, there is no known method for automatic integration that can directly handle integral operators within the same framework. This presents a significant challenge when dealing with problems that involve integral operators, as these require the computation of integrals that cannot be automatically differentiated. To address this limitation in PINN frameworks, researchers have developed and employed various techniques to approximate these integrals. Numerical integration methods, such as Newton-Cotes, Gaussian quadrature, and Monte Carlo integration, are among the most commonly used approaches for solving integral equations. In the context of optimal control problems, which often involve integral cost functions, additional mathematical methods, such as the Hamiltonian or Euler-Lagrange formulations, are sometimes employed to reformulate the problem and eliminate the need for direct integration. A comprehensive review of methods for solving integral equations and optimal control problems is summarized in Tables \ref{tbl:review-ie} and \ref{tbl:review-oc}, respectively.

In the papers reviewed in these tables, beyond the standard ordinal and partial derivatives that can be computed using automatic differentiation, there is another type of differentiation known as fractional differentiation, where the order of the derivative is not an integer. Fractional derivatives and integrals, which fall under the broader umbrella of fractional calculus, represent a newly developed and rapidly growing branch of mathematics \cite{mainardi2022fractional}. Their appeal lies in their flexibility, allowing for more accurate modeling of complex, real-world phenomena that cannot be adequately captured by traditional integer-order calculus. For instance, fractional calculus has been applied in areas such as viscoelastic material modeling, where it can describe the material's behavior more precisely than classical models. Similarly, in signal processing, fractional differentiation helps in enhancing signal accuracy and filtering, providing more refined tools for analysis \cite{bagley1983theoretical, METZLER20001}.

In fractional calculus, the definitions of derivatives and integrals are not unique, and each problem may require a specific fractional derivative or integral definition tailored to its particular context. However, most of these definitions are expressed in terms of an integral operator, making their computation particularly challenging. Consequently, the papers reviewed often employ approximations to fractional derivatives when training neural networks, allowing for more practical implementation despite the inherent complexity of fractional calculus.

In this paper, we aim to develop an efficient method for accurately approximating integral operators, as well as fractional derivative operators, for solving mathematical problems involving these operators. Specifically, our contributions are as follows:

\begin{itemize}
    \item Proposing an efficient matrix-vector and tensor-vector product method for approximating integral operators using Gaussian quadrature techniques.
    \item Solving a range of Fredholm and Volterra integral and integro-differential equations, including multi-dimensional and systems of equations.
    \item Solving the well-known Volterra's population model of fractional order.
    \item Solving a set of different optimal control problems, including those with delay terms, integro-differential constraints, and fractional derivatives.
    \item Solving inverse integro-differential equations of Fredholm and Volterra types with potential fractional derivative orders.
    \item Developing a Python package for easy implementation of the proposed methods.
\end{itemize}

Our proposed method significantly advances the state-of-the-art in solving integral and integro-differential equations using physics-informed neural networks. While much of the existing literature has primarily focused on forward problems \cite{effati2012neural,mosleh2014numerical,jafarian2015artificial,asady2014utilizing,chaharborj2017study,guo2021physics,lu2021deepxde,guan2022solving,zhang2022physics,saneifard2022extended,martire2022fractional,lu2023approximate,allahviranloo2023application,shao2023feedforward,ruocco2023efficient,sun2023binn,firoozsalari2023deepfdenet,zhang2024solution,jiang2024deep,georgiou2024fredholm}, our approach addresses both forward and inverse problems, aligning with recent efforts by \cite{yuan2022pinn, zhang2024pinn, saadat2024unifides, li2024machine}. However, we extend these capabilities to a broader spectrum of equations, including fractional, Volterra, Fredholm, systems of equations, and multi-dimensional integro-differential equations.

A key distinguishing feature of our method is its ability to handle problems on infinite intervals and those with singularities, which are not explicitly addressed in most previous works. This capability, combined with our treatment of fractional operators, sets our approach apart from earlier studies such as \cite{chaharborj2017study, saneifard2022extended, martire2022fractional, allahviranloo2023application, saadat2024unifides}, which were limited to specific types of fractional equations.

Our use of Gaussian quadrature for numerical integration provides a more robust and versatile approach compared to methods relying on analytical integration \cite{effati2012neural, jafarian2015artificial, chaharborj2017study, lu2023approximate, allahviranloo2023application, saneifard2022extended} or simpler numerical techniques like the trapezoidal rule \cite{saadat2024unifides, jiang2024deep}. While some recent works such as \cite{asady2014utilizing, martire2022fractional, shao2023feedforward, firoozsalari2023deepfdenet, sun2023binn} have also employed Gaussian quadrature, our method applies it more comprehensively across a wider range of problem types. Furthermore, unlike approaches that utilize specialized neural network architectures such as fuzzy neural networks or extreme learning machines \cite{mosleh2014numerical, lu2023approximate, bassi2024learning}, our method employs standard Multi-Layer Perceptrons, making it more accessible and easier to implement within existing deep learning frameworks.

Moreover, our framework for solving optimal control problems represents a significant advancement in the field, offering a more comprehensive and versatile approach compared to existing methods. While several previous works have addressed optimal control problems using physics-informed neural networks, our approach stands out in several key aspects. Our method tackles a wide range of optimal control problems, including those with delay terms, integro-differential constraints, and fractional derivatives. This breadth of application surpasses many previous works that focused on specific types of problems, such as those dealing only with delay differential equations \cite{kheyrinataj2020fractional, kheyrinataj2023solving}, or those limited to fractional optimal control problems \cite{ghasemi2017nonlinear, sabouri2017neural, yavari2019fractional, kheyrinataj2020fractional, kheyrinataj2023solving}. We employ Gaussian quadrature for numerical integration, which offers higher accuracy and flexibility compared to methods using simpler techniques like the midpoint rule \cite{mowlavi2023optimal} or those relying on Simpson's rule \cite{kheyrinataj2023solving,sabouri2017neural}. Our approach effectively handles complex constraints, including those involving integral operators and fractional derivatives, a capability not explicitly addressed in many previous works. Furthermore, our method is capable of solving multi-dimensional optimal control problems, a feature not explicitly mentioned in several of the listed works.

Finally, our development of a Python package for implementing these methods enhances the accessibility and usability of our approach. While packages like DeepXDE exist, our solution offers greater flexibility for high-dimensional and fractional integro-differential equations. A detailed comparison with DeepXDE will be presented later in the paper to demonstrate the effectiveness and advantages of our approach.

In the following sections, we first discuss numerical integration methods, placing a particular emphasis on Gaussian quadrature techniques. Subsequently, in Section 3, we introduce the method for approximating various integral and fractional operators. Section 4 presents a series of experiments conducted on forward and inverse integral equations, as well as on optimal control problems. In Section 5, we showcase our developed Python package designed for solving integral equations. Finally, Section 6 offers concluding remarks.

\begin{table}[ht]
\resizebox{\textwidth}{!}{%
\begin{tabular}{@{}cccccc@{}}
\toprule
Paper & Year & Task & Equation Type & Model & Integration Technique \\ \midrule
\cite{effati2012neural} & 2012 & Fwd. & Fredholm IE & MLP & Analytical \\
\cite{mosleh2014numerical} & 2014 & Fwd. & Fuzzy Fredholm IDE & FNN & Newton-Cotes \\
\cite{jafarian2015artificial} & 2014 & Fwd. & Systems of Volterra IE & MLP & Analytical \\
\cite{asady2014utilizing} & 2014 & Fwd. & 2D Fredholm IE & MLP & Gauss-Legendre \\
\cite{chaharborj2017study} & 2017 & Fwd. & Fractional Volterra IDE & MLP & Analytical \\
\cite{guo2021physics} & 2021 & Fwd. & Fredholm IE & MLP & Method of Moments \\
\cite{lu2021deepxde} & 2021 & Fwd. & Volterra IE & MLP & Gauss-Legendre \\
\cite{guan2022solving} & 2022 & Fwd. & Fredholm IE & MLP & Monte Carlo \\
\cite{zhang2022physics} & 2022 & Fwd. & Fredholm IE& MLP & Method of Moments \\
\cite{yuan2022pinn} & 2022 & Fwd., Inv. & \begin{tabular}[c]{@{}c@{}}Volterra, Fredholm, Systems\\ Multi-Dimensional IDE\end{tabular} & MLP & Auxiliary Variable \\
\cite{saneifard2022extended} & 2022 & Fwd. & Fractional 2D Volterra IDE & MLP & Analytical \\
\cite{martire2022fractional} & 2022 & Fwd. & Fractional Volterra IE& MLP & Gauss-Legendre \\
\cite{lu2023approximate} & 2023 & Fwd. & Volterra, Fredholm IE& ELM & Analytical \\
\cite{allahviranloo2023application} & 2023 & Fwd. & Fractional IDE & MLP & Analytical \\
\cite{shao2023feedforward} & 2023 & Fwd. & Systems of Partial IDE & MLP & Gauss-Legendre \\
\cite{ruocco2023efficient} & 2023 & Fwd. & Fredholm IE & MLP & Gaussian \\
\cite{sun2023binn} & 2023 & Fwd. & Fredholm IE & MLP & Piece-wise Gaussian \\
\cite{firoozsalari2023deepfdenet} & 2023 & Fwd. & Volterra IE & MLP & Gauss-Legendre \\
\cite{zhang2024solution} & 2024 & Fwd. & Fredholm IE& MLP & Method of Moments \\
\cite{zhang2024pinn} & 2024 & Fwd., Inv. & Multi-Dimensional Fredholm IDE & MLP & Gauss-Legendre \\
\cite{saadat2024unifides} & 2024 & Fwd., Inv. & \begin{tabular}[c]{@{}c@{}}Fractional Fredholm, Volterra,\\ Systems, IDE\end{tabular} & MLP & Trapezoidal \\
\cite{jiang2024deep} & 2024 & Fwd. & Fredholm IE & MLP & Compound Trapezoidal \\
\cite{li2024machine} & 2024 & Fwd., Inv. & \begin{tabular}[c]{@{}c@{}}Volterra, Fredholm, Systems,\\ Multi-Dimensional, IDE\end{tabular} & MLP & Auxiliary Variable \\
\cite{georgiou2024fredholm} & 2024 & Fwd. & Fredholm IE & MLP & Fixed Point Iteration \\
Ours & 2024 & Fwd., Inv. & \begin{tabular}[c]{@{}c@{}}Fractional, Volterra, Fredholm, Systems\\ Multi-Dimensional IDE, \\ Infinite Interval, Singular\end{tabular} & MLP & Gaussian Quadrature \\
\bottomrule
\end{tabular}%
}
\caption{Comparison of data-driven approaches for solving various types of integral and integro-differential equations in recent years. Most methods focus primarily on forward problems, particularly with fractional derivatives. Some methods utilize symbolic computation in numerical software to evaluate integrals, while others employ numerical integration techniques like Newton-Cotes and Gaussian quadrature, as implemented in Extreme Learning Machines (ELM), Fuzzy Neural Networks (FNN), Multi-Layer Perceptrons (MLP), or Recurrent Neural Networks (RNN). In Section 2 of the paper, we will explain these numerical integration methods.}
\label{tbl:review-ie}
\end{table}

\begin{table}[ht]
\resizebox{\textwidth}{!}{%
\begin{tabular}{@{}cccccc@{}}
\toprule
Authors & Year & Time & Equation Type & Model & Integration Technique \\ \midrule
\cite{becerikli2003intelligent} & 2003 & Continuous & Ordinal & MLP & Runge-Kutta-Butcher \\
\cite{liu2012neural} & 2012 & Discrete & Ordinal & MLP & - \\
\cite{effati2013optimal} & 2013 & Continuous & Ordinal & MLP & Hamiltonian \\
\cite{ghasemi2017nonlinear} & 2017 & Continuous & Nonlinear, Fractional & MLP & Hamiltonian \\
\cite{sabouri2017neural} & 2017 & Continuous & Fractional & MLP & Simpson's Rule \\
\cite{yavari2019fractional} & 2019 & Continuous & Fractional, Infinite-Horizon & MLP & Hamiltonian \\
\cite{mortezaee2020infinite} & 2019 & Continuous & Ordinal, Infinite-Horizon & MLP & Hamiltonian \\
\cite{kheyrinataj2020fractional} & 2020 & Continuous & Delay, Fractional & MLP & Hamiltonian \\
\cite{benning2021deep} & 2021 & Continuous & Ordinal & - & Hamiltonian \\
\cite{bottcher2022ai} & 2022 & Continuous & Ordinal & MLP & Hamiltonian \\
\cite{barry2022physics} & 2022 & Continuous & Partial & MLP & Hamiltonian \\
\cite{chen2018optimal} & 2018 & Discrete & Ordinal & RNN & - \\
\cite{sanchez2018real} & 2018 & Continuous & Ordinal & MLP & Numerical Integration \\
\cite{yin2023optimal} & 2023 & Continuous & Ordinal & MLP & Gaussian Quadrature \\
\cite{dai2023solving} & 2023 & Continuous & Ordinal & MLP & Hamiltonian \\
\cite{mowlavi2023optimal} & 2023 & Continuous & Partial Differential & MLP & Midpoint Rule \\
\cite{kheyrinataj2023solving} & 2023 & Continuous & Delay, Fractional & MLP & Simpson's Rule \\
\cite{nzoyem2023comparison}  &  2023 & Continuous & Partial & MLP & Adjoint Method \\
\cite{na2024physics} & 2024 & Continuous & Ordinal & MLP & Euler–Lagrange \\
\cite{yin2024aonn}  & 2024 & Continuous & Partial & MLP & Hamiltonian \\
Ours & 2024 & Continuous & \begin{tabular}[c]{@{}c@{}}Ordinal, Partial, Fractional,\\ Multi-Dimensional, IDE, Delay\\ Nonlinear\end{tabular} & MLP & Gaussian Quadrature \\
\bottomrule
\end{tabular}%
}
\caption{A review of neural network methods for solving optimal control problems. Due to the absence of automatic integration, most of these studies leveraged the mathematical properties of the systems, formulating the problems using Hamiltonian or Euler–Lagrange methods to eliminate the need for integration. Others applied numerical integration techniques for simpler problems.}
\label{tbl:review-oc}
\end{table}

\section{Gaussian Quadrature}

Numerical quadrature is a fundamental technique in numerical analysis for approximating the definite integral of a function. Various methods have been developed to address this problem, which can be broadly classified into Newton-Cotes methods, Gaussian quadrature, statistical techniques, and adaptive strategies. Newton-Cotes methods, including the trapezoidal rule and Simpson's rule, are among the earliest approaches for estimating the integral of a function over a specified domain. Gaussian quadrature, while similar in approach to Newton-Cotes formulas, offers more precise approximations. Statistical methods, such as the Monte Carlo method and Bayesian quadrature, approximate the integral by adopting a probabilistic perspective, accounting for uncertainty in the solution. Adaptive methods represent another strategy for approximating integrals, particularly for stiff or irregular functions, by dividing the integration into subintervals and applying the aforementioned static methods.

Among these methods, Newton-Cotes and Gaussian quadrature possess robust mathematical foundations and theoretical support. In these approaches, the integral of a function \( u(x) \) over the interval \([a, b]\) is approximated by a linear weighted sum:
\[
\mathcal{I}(u) = \int_{a}^{b} u(x) \omega(x) \, dx \approx \mathbf{u}^\top \mathbf{w} = \sum_{i=0}^{n} w_i u(x_i),
\]
where \(n\) represents the number of quadrature points, \(\omega(x)\) is a weight function, \(w_i\) are the quadrature weights, and \(x_i \in [a, b]\) are the nodes. The Newton-Cotes method assumes \(\omega(x) = 1\) and that the nodes \(x_i\) are equally spaced, specifically \(x_i = a + ih\), where \(h = (b-a)/n\). The weights \(w_i\) are computed by analytically integrating the Lagrange basis polynomials. However, this approach has certain limitations: the nodes \(x_i\) must be equidistant, and the use of Lagrange polynomials makes the method vulnerable to Runge's phenomenon, potentially leading to inaccurate integration. Additionally, this method is proven to be exact only for polynomials of degree at most \(n\).

Conversely, Gaussian quadrature allows both the weights and nodes to be variable, increasing the degrees of freedom to \(2n+2\). This flexibility enables the method to compute the weights and nodes in such a way that the integration is exact for polynomials of degree up to \(2n+1\). In determining these unknowns, a relationship between the weights and nodes and orthogonal polynomials emerges. For instance, in Gaussian quadrature with \([a, b] = [-1, 1]\) and \(\omega(x) = 1\), the nodes \(x_i\) are the roots of the Legendre polynomials, and the corresponding weights \(w_i\) can be computed accordingly. When dealing with a finite interval other than \([-1, 1]\), the same method can be applied with a simple transformation:
\begin{equation*}
    \mathcal{I}(u) = \int_{a}^{b} u(x)\omega(x) \, dx \approx \frac{b-a}{2} \sum_{i=0}^{n} w_{i} u\left(\frac{b-a}{2}x_{i} + \frac{a+b}{2}\right).
\end{equation*}

Similarly, for various orthogonal polynomials associated with their respective weight functions \(\omega(x)\), the nodes \(x_i\) and weights \(w_i\) can be systematically determined. Table \ref{tbl:gauss_quadrature} provides a summary of some prominent Gaussian quadrature rules for integrations over finite, semi-infinite, and infinite domains. This table outlines various orthogonal polynomials, each linked to specific weight functions and integration domains. Included are Jacobi polynomials, along with their well-known successors: Legendre polynomials and the four primary types of Chebyshev polynomials. Additionally, Laguerre and Hermite polynomials are covered. These polynomials are widely utilized in Gaussian quadrature methods because their orthogonality properties ensure optimal node placements and weight calculations, facilitating accurate numerical integration. The following section outlines the definitions and key characteristics of these functions.

\paragraph{Jacobi Polynomials}
The Jacobi polynomials \( J_n^{(\alpha, \beta)}(x) \) are defined by the explicit formula \cite{aghaei2024rkan,10454075}:

\[
J_n^{(\alpha, \beta)}(x) = \frac{(\alpha+1)_n}{n!} \sum_{k=0}^{n} \binom{n}{k} \frac{(\alpha+\beta+n+1)_k}{(\alpha+1)_k} \left(\frac{x-1}{2}\right)^k,
\]
where \(\alpha, \beta > \mathbb{R}^{>-1}\) are Jacobi parameters, and \( (\alpha+1)_n \) denotes the Pochhammer symbol, which represents the rising factorial:
\[
(\alpha+1)_n = (\alpha+1)(\alpha+2)\cdots(\alpha+n).
\]
The derivatives of these functions can be computed using the following formula:
\begin{equation*}
    \frac{d}{dx} J_n^{(\alpha, \beta)}(x) = \frac{1}{2} (n + \alpha + \beta + 1) J_{n-1}^{(\alpha+1, \beta+1)}(x).
\end{equation*}
Legendre and Chebyshev polynomials are special cases of Jacobi polynomials in which \(\alpha\) and \(\beta\) take on specific values. Specifically, Legendre polynomials are given by \( P_n(x) = J_n^{(0,0)}(x) \), Chebyshev polynomials of the first kind are \( T_n(x) = J_n^{\left(-\frac{1}{2},-\frac{1}{2}\right)}(x) \), Chebyshev polynomials of the second kind are \( U_n(x) = J_n^{\left(\frac{1}{2},\frac{1}{2}\right)}(x) \), Chebyshev polynomials of the third kind are \( V_n(x) = J_n^{\left(-\frac{1}{2},\frac{1}{2}\right)}(x) \), and Chebyshev polynomials of the fourth kind are \( W_n(x) = J_n^{\left(\frac{1}{2},-\frac{1}{2}\right)}(x) \).

In all cases, an orthogonality relationship can be observed between each pair of Jacobi polynomials, described using the Euclidean inner product. Specifically, the orthogonality condition for Jacobi polynomials \(J^{(\alpha,\beta)}_n\) and \(J^{(\alpha,\beta)}_m\), as well as their successors, is given by:
\[
\langle J^{(\alpha,\beta)}_n, J^{(\alpha,\beta)}_m \rangle = \int_{-1}^{1} J^{(\alpha,\beta)}_n(x) J^{(\alpha,\beta)}_m(x) (1-x)^\alpha (1+x)^\beta \, dx = \langle J^{(\alpha,\beta)}_n, J^{(\alpha,\beta)}_n \rangle \delta_{m,n},
\]
where \(\delta_{m,n}\) is the Kronecker delta function, defined as \(1\) if \(m = n\), and \(0\) otherwise.

\paragraph{Laguerre Polynomials}

The Laguerre polynomials, \( L_n(x) \), form a sequence of orthogonal polynomials that arise in quantum mechanics, particularly in the radial part of the solution to the Schrödinger equation for a hydrogen-like atom. They are also utilized in various approximation methods and numerical analysis. The generalized Laguerre polynomials, \( L_n^{(\alpha)}(x) \), extend this concept by introducing a parameter \( \alpha \), allowing for more flexible applications and encompassing a broader range of problems. For \(\alpha > -1\), these generalized Laguerre polynomials are defined by the explicit formula:
\[
L_n^{(\alpha)}(x) = \sum_{k=0}^{n} (-1)^k \binom{n+\alpha}{n-k} \frac{x^k}{k!},
\]
and possess the orthogonality property over the semi-infinite domain:
\[
\langle L^{(\alpha)}_n, L^{(\alpha)}_m \rangle = \int_{0}^\infty L^{(\alpha)}_n(x) L^{(\alpha)}_m(x) \mathrm{e}^{-x} dx = \langle L^{(\alpha)}_n, L^{(\alpha)}_n \rangle \delta_{m,n}.
\]
The derivatives of these polynomials can be efficiently computed using:
\begin{equation*}
    \frac{d}{dx} L_n^{(\alpha)}(x) = -L_{n-1}^{(\alpha+1)}(x).
\end{equation*}

\paragraph{Hermite Polynomials}

The Hermite polynomials, \( H_n(x) \), are a sequence of orthogonal polynomials that play significant roles in probability theory, combinatorics, and quantum mechanics, where they serve as the eigenfunctions of the quantum harmonic oscillator. The explicit formula for Hermite polynomials is given by:
\[
H_n(x) = n! \sum_{k=0}^{\lfloor n/2 \rfloor} \frac{(-1)^k (2x)^{n-2k}}{k!(n-2k)!}.
\]
Their derivatives can be computed using the following relation:
\begin{equation*}
\frac{d}{dx} H_n(x) = 2n H_{n-1}(x).
\end{equation*}
These polynomials, defined on the real line, possess the following orthogonality property with respect to the weight function \(\omega(x) = \exp(-x^2)\):
\[
\langle H_n, H_m \rangle = \int_{-\infty}^\infty H_n(x) H_m(x) \mathrm{e}^{-x^2} dx = \langle H_n, H_n \rangle \delta_{m,n}.
\]

\begin{table}[ht]
\resizebox{\textwidth}{!}{%
\begin{tabular}{@{}ccccc@{}}
\toprule
\textbf{Method}                                                         & \textbf{Domain}         & \textbf{Weight} \( \omega(x) \) & \textbf{Roots} \( x_i \)                   & \textbf{Weights} \( w_i \)                                                                                                                                                                    \\ \midrule

Gauss-Legendre                                                          & \( [-1, 1] \)           & 1                               & Zeros of \( P_n(x) \)                      & \( \frac{2}{(1-x_i^2)[P'_n(x_i)]^2} \)                                                                                                                                                  \\
\begin{tabular}[c]{@{}c@{}}Gauss-Chebyshev\\ (First Kind)\end{tabular}  & \( (-1, 1) \)& \( \frac{1}{\sqrt{1-x^2}} \)    & \( \cos\left(\frac{2i-1}{2n} \pi\right) \) & \( \frac{\pi}{n} \) \\
\begin{tabular}[c]{@{}c@{}}Gauss-Chebyshev\\ (Second Kind)\end{tabular} & \( [-1, 1] \)           & \( \sqrt{1-x^2} \)              & \( \cos\left(\frac{i}{n+1} \pi\right) \)   & \( \frac{\pi}{n+1} \sin^2\left(\frac{i\pi}{n+1}\right) \) \\

\begin{tabular}[c]{@{}c@{}}Gauss-Chebyshev\\ (Third Kind)\end{tabular}  & \( [-1, 1] \)           & \( \sqrt{1+x}\sqrt{1-x^3} \)    & \( \cos\left(\frac{2i-1}{2n} \pi\right) \) & \( \frac{\pi}{n} \sqrt{1+x_i} \)                                                                                                                                                        \\
\begin{tabular}[c]{@{}c@{}}Gauss-Chebyshev\\ (Fourth Kind)\end{tabular} & \( [-1, 1] \)           & \( \sqrt{1-x}\sqrt{1-x^3} \)    & \( \cos\left(\frac{2i-1}{2n} \pi\right) \) & \( \frac{\pi}{n} \sqrt{1-x_i} \)                                                                                                                                                        \\
Gauss-Jacobi                                                            & \( (-1, 1) \)& \( (1-x)^\alpha (1+x)^\beta \)  & Zeros of \( J_n^{(\alpha, \beta)}(x) \)& \(\frac{2^{\alpha+\beta+1} 2^{\alpha+\beta}\Gamma(n+\alpha+1)\Gamma(n+\beta+1)}{(2n+\alpha+\beta+1)J^{\prime}_n(x_i)J^{\prime}_n(x_i)\Gamma(n+\alpha+\beta+1)n!}\) \\

Gauss-Laguerre                                                          & \( [0, \infty) \)       & \( x^\alpha e^{-x} \)           & Zeros of \( L_n^{(\alpha)}(x) \)           & \( \frac{x_i^{\alpha+1}}{[(n+1) L_{n+1}^{(\alpha)}(x_i)]^2} \)                                                                                                                          \\ 
Gauss-Hermite                                                           & \( (-\infty, \infty) \) & \( e^{-x^2} \)                  & Zeros of \( H_n(x) \)                      & \( \frac{2^{n-1} n! \sqrt{\pi}}{n^2 [H_{n-1}(x_i)]^2} \)                                                                                                                                \\
\bottomrule
\end{tabular}
}
\caption{Summary of various Gaussian quadrature methods, detailing the domain of integration, the weight functions \( \omega(x) \), the roots \( x_i \) of the respective orthogonal polynomials, and the corresponding weights \( w_i \) for each method. The table includes well-known methods such as Gauss-Legendre, Gauss-Chebyshev (of the first, second, third, and fourth kinds), Gauss-Jacobi, Gauss-Laguerre, and Gauss-Hermite, each optimized for specific weight functions and domains.}
\label{tbl:gauss_quadrature}

\end{table}

\section{Methodology}
\label{sec:3}
This section describes how a mathematical problem involving an integral operator can be solved using a deep learning architecture. Initially, we outline the method for approximating the solution using an MLP neural network architecture. Following this, we define the physics-informed loss function and explain how the integral component of the equation can be computed using Gaussian quadrature in a vectorized form. The section concludes with a discussion on the application of the proposed approach for approximating fractional derivatives, along with a matrix-vector product method for accurately predicting the Caputo fractional derivative.

\subsection{Approximating the Solution}
In this section, we consider a general form of a mathematical equation represented in an operator form:
\begin{equation}
    \mathcal{F}(u)(\mathbf{x}) + \mathcal{D}(u)(\mathbf{x}) + \mathcal{I}(u)(\mathbf{x}) = \mathcal{S}(\mathbf{x}),
    \label{eq:ie}
\end{equation}
given by some initial and boundary conditions, where $u(\mathbf{x})$ is the unknown solution, $\mathcal{F}(u)$ represents an algebraic function of $u$, $\mathcal{D}(u)$ is a differential operator acting on $u$, $\mathcal{I}(u)$ is an integral operator applied to $u$, and $\mathcal{S}(\mathbf{x})$ is the source function, which is defined by the independent variable $\mathbf{x} \in \mathbb{R}^d$ for a $d$-dimensional problem.

To approximate the solution to this problem, one may consider the unknown solution $u(\mathbf{x})$ by an MLP neural network:
\begin{equation*}
    \begin{aligned}
        \mathcal{A}_0 &= \mathbf{X}, & \mathbf{X} \in \mathbb{R}^{N\times d}, \\
        \mathcal{A}_i &= \sigma_i(\mathcal{A}_{i-1}\boldsymbol{\theta}^{(i)}+\mathbf{b}^{(i)}), & i = 1,2,\ldots L-1,\\
        \mathrm{MLP}(\mathbf{X}) &:= \mathcal{A}_L = \mathcal{A}_{L-1}\boldsymbol{\theta}^{(L)}+\mathbf{b}^{(L)}, & \mathcal{A}_L \in \mathbb{R}^{N\times 1}.
    \end{aligned}
\end{equation*}
Here, $\mathbf{X} \in \mathbb{R}^{N\times d}$ represents a set of $N$ training points, or collocation points, in a $d$-dimensional space, defined within the problem domain $\Delta = [a,b]$. The network weights are denoted by $\boldsymbol{\theta} \in \mathbb{R}^{h_{i-1}\times h_i}$, where $h_i$ is the number of neurons in the $i^\text{th}$ layer. The output of the $i^{\text{th}}$ layer is given by $\mathcal{A}_i$, and $\sigma_i(\cdot)$ is the activation function applied in that layer, with $L$ representing the total number of layers. Considering the output of the network, $\mathbf{u}=\mathrm{MLP}(\mathbf{X})$, that is an $N\times 1$ vector containing the approximated function in $N$ training points, one can define a residual function $\mathfrak{R}(\mathbf{X})\in\mathbb{R}^{N\times 1}$ that measures how the approximated function fits the dynamics of the problem:
\begin{equation*}
    \mathfrak{R}(\mathbf{X}) := \mathcal{F}(\mathbf{u}) + \mathcal{D}(\mathbf{u}) + \mathcal{I}(\mathbf{u}) - \mathcal{S}(\mathbf{X}).
\end{equation*}
Then, the loss function of the network should be constructed in such a way that the gradient descent algorithm minimizes the absolute value of the residual, leading to a more accurate prediction. To achieve this, the loss function can be defined in a supervised learning framework as follows:
\begin{equation}
    \mathcal{L}(\mathbf{X}) = \frac{1}{N} \mathfrak{R}(\mathbf{X})^\top \mathfrak{R}(\mathbf{X}) + \lambda^{\text{IC}} \mathrm{MSE}^{\text{IC}} + \lambda^{\text{BC}} \mathrm{MSE}^{\text{BC}} + \lambda^{\text{Data}} \mathrm{MSE}^{\text{Data}},
    \label{eq:loss}
\end{equation}
where $\mathrm{MSE}^{\text{IC}}, \mathrm{MSE}^{\text{BC}}, \mathrm{MSE}^{\text{Data}}$ are the mean squared errors (MSE) between the predicted solution by the network and the initial conditions, boundary conditions, and real-world data, respectively. The terms $\lambda^{\text{IC}}, \lambda^{\text{BC}}, \lambda^{\text{Data}}$ are positive coefficients that serve as regularization parameters.

A similar approach can be employed to solve systems of mathematical equations of the form:
\begin{equation*}
    \mathcal{F}_\iota(U)(\mathbf{x}) + \mathcal{D}_\iota(U)(\mathbf{x}) + \mathcal{I}_\iota(U)(\mathbf{x}) = \mathcal{S}_\iota(\mathbf{x}),
\end{equation*}
where $\iota = 1, 2, \ldots, M$ and $M$ is the number of equations, with $U = [u_1, u_2, \ldots, u_M]$ being the set of unknown functions in the system. In this context, the residual functions are denoted by $\mathfrak{R}_\iota(\cdot)=\mathcal{F}_\iota(\mathbf{U})(\cdot) + \mathcal{D}_\iota(\mathbf{U})(\cdot) + \mathcal{I}_\iota(\mathbf{U})(\cdot) - \mathcal{S}_\iota(\cdot)$, where $\mathbf{U}=[\mathbf{u}_1, \mathbf{u}_2, \ldots, \mathbf{u}_M]$ and $\mathbf{u}_\iota = \mathrm{MLP}_\iota(\mathbf{X})$ is the predicted solution for the $\iota^\text{th}$ equation, generated by one of the $M$ different neural network architectures. The overall loss function can then be defined as:
\begin{equation*}
    \mathcal{L}(\mathbf{X}) = \frac{1}{N \times M} \sum_{\iota=1}^M \mathfrak{R}_\iota(\mathbf{X})^\top \mathfrak{R}_\iota(\mathbf{X}) + \sum_{\iota=1}^M \left(\lambda_\iota^{\text{IC}} \mathrm{MSE}_\iota^{\text{IC}} + \lambda_\iota^{\text{BC}} \mathrm{MSE}_\iota^{\text{BC}} + \lambda_\iota^{\text{Data}} \mathrm{MSE}_\iota^{\text{Data}}\right).
\end{equation*}

In both scenarios, minimizing the loss function, or determining the optimal weights $\boldsymbol{\theta}$, can be achieved using a variety of optimization algorithms, such as gradient descent (including its variants like RMSprop, Momentum, Adam), conjugate gradient, Levenberg-Marquardt, or the Broyden–Fletcher–Goldfarb–Shanno (BFGS) methods.

For this study, we employ the Limited-memory BFGS (L-BFGS) algorithm, a quasi-Newton optimizer known for its accuracy and efficiency. L-BFGS is a widely used optimization method that offers the benefits of the BFGS algorithm while being more memory-efficient. Unlike the traditional BFGS, which requires handling dense matrices with a quadratic scaling in relation to the number of parameters, L-BFGS retains only a limited number of vectors to capture curvature information. This feature is particularly beneficial for large-scale problems, as it significantly reduces memory consumption while preserving the quasi-Newton characteristic of approximating the inverse Hessian matrix \cite{nocedal1999numerical}.

L-BFGS operates by iteratively updating an initial solution vector $\boldsymbol{\theta}_k$ using gradient information and a limited memory of past updates. The algorithm starts with an initial guess $\boldsymbol{\theta}_0$ and iteratively refines it. In each iteration, the gradient $\nabla \mathcal{L}(\boldsymbol{\theta}_k)$ is computed, and the inverse Hessian-vector product is approximated using a limited history of previous gradient differences $\mathbf{y}_j = \nabla \mathcal{L}(\boldsymbol{\theta}_j) - \nabla \mathcal{L}(\boldsymbol{\theta}_{j-1})$ and position differences $\mathbf{s}_j = \boldsymbol{\theta}_j - \boldsymbol{\theta}_{j-1}$ for $j = k-m+1, \ldots, k$. The search direction $\mathbf{p}_k$ is then computed using a two-loop recursion as
$$\mathbf{p}_k = -\mathbf{H}_k \nabla \mathcal{L}(\boldsymbol{\theta}_k),$$
where $\mathbf{H}_k$ represents the implicit inverse Hessian approximation. A line search is conducted to determine the optimal step size $\alpha_k$, and the parameter vector is updated as
$$\boldsymbol{\theta}_{k+1} = \boldsymbol{\theta}_k + \alpha_k \mathbf{p}_k.$$
In this scenario, the derivative of the loss function with respect to the weights, denoted as $\nabla \mathcal{L}(\boldsymbol{\theta}_k)$, is efficiently computed using the backpropagation algorithm, which is a specific implementation of reverse-mode automatic differentiation developed for neural networks. This technique utilizes the chain rule from calculus:
\begin{equation*}
    \begin{aligned}
        \frac{\partial \mathcal{L}}{\partial \boldsymbol{\theta}^{(i)}} &= \delta^{(i)} \cdot \left(\mathcal{A}_{i-1}\right)^\top,\\
        \delta^{(i)} &= \left(\boldsymbol{\theta}^{(i+1)}\right)^\top \delta^{(i+1)} \odot \sigma_i'\left(\mathcal{A}_{i-1} \boldsymbol{\theta}^{(i)} + \mathbf{b}^{(i)}\right),
    \end{aligned}
\end{equation*}
where $\odot$ represents the Hadamard element-wise product. A similar approach can be applied to the network for computing the derivatives of the learned function $u(\mathbf{X}) = \mathrm{MLP}(\mathbf{X})$ with respect to $\mathbf{X}$. This makes the calculation of the differential operator $\mathcal{D}(\cdot)$ straightforward. However, there is currently no widely adopted automatic integration tool for accurately computing the integral term $\mathcal{I}(\cdot)$. In the following section, we present a matrix-vector and tensor-vector product approach for efficiently computing integral operators. A high-level overview of the PINN framework for constructing the loss function in problems involving integral operators and fractional derivatives (as defined in Eq. \eqref{eq:ie}) is presented in Figure \ref{fig:schema}.

\begin{figure}
    \centering
    \includegraphics[width=0.7\textwidth]{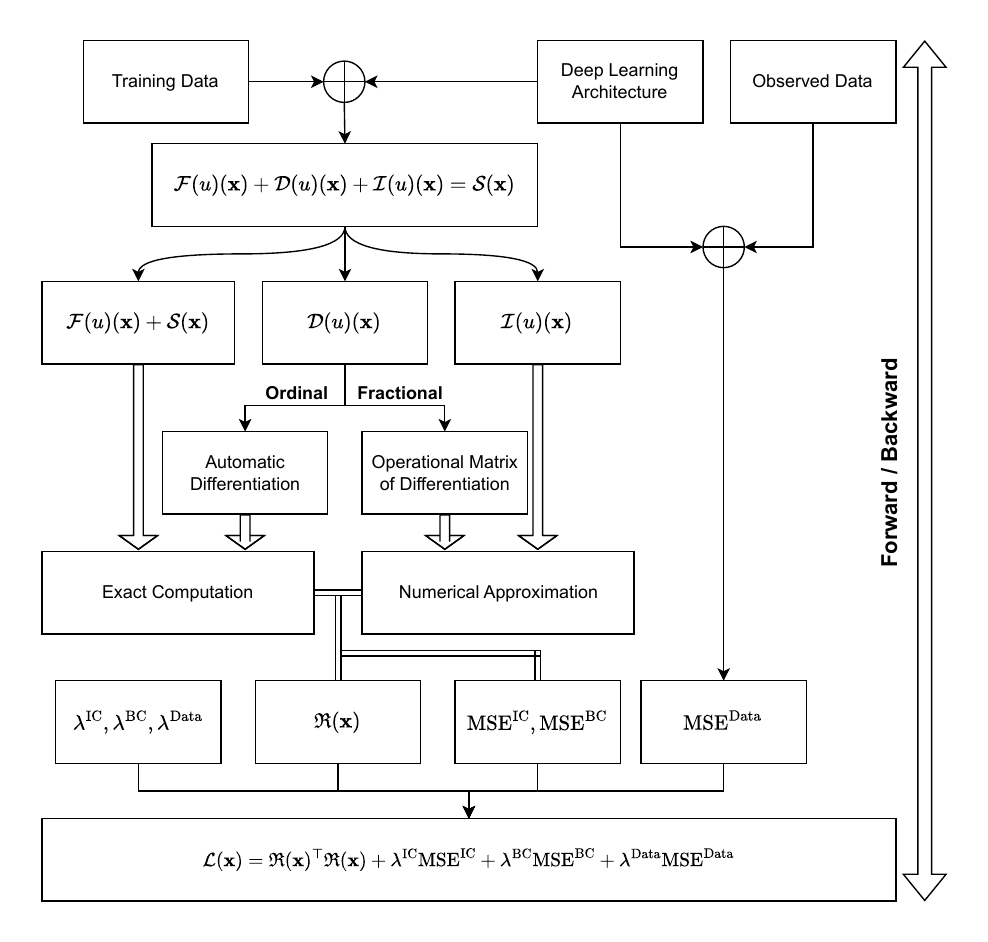}
    \caption{Diagram illustrating the proposed methodology for solving functional equations that include ordinary, partial, and fractional derivatives, as well as integral operators.}
    \label{fig:schema}
\end{figure}

\subsection{Fredholm Integral Operator}
Consider a one-dimensional Fredholm integral operator with a kernel function $\mathcal{K}(x,t)$ defined as:
\begin{equation}
    \mathcal{I}(u)(x) = \int_a^b \mathcal{K}(x,t) u(t) \, dt,
    \label{eq:fredholm_operator}
\end{equation}
where $x \in [a,b]$. Depending on the structure of the kernel function and the integration domain, one should select an appropriate method from Table \ref{tbl:gauss_quadrature} to compute this integral. Generally, the Gauss-Legendre algorithm is effective for any kernel function, although it may be less accurate than Gauss-Chebyshev for singular kernels. In any case, the computation of the Gaussian integration weights ($\mathbf{w}$) and roots ($\mathbf{r}$) is independent of the network training phase, allowing these values to be precomputed during the initialization of the architecture.

The roots $\mathbf{r}$, serving as a surrogate for the integration variable $t$, should be used to approximate the integral by evaluating the integrands $u(t)$ and $\mathcal{K}(x,t)$ at these points. For the independent variable $x$, the network training data $\mathbf{x}$ should be used. These two vectors may be identical, i.e., $\mathbf{x}=\mathbf{r}$, meaning the roots of the orthogonal polynomials are used as training data. In this case, the procedure will be very fast, as the forward phase $u(\mathbf{r})$ needs to be called only once. In any scenario, the function $I(x)$ can be defined to compute the integrand for a given $x$:
\begin{equation}
    I(x) = \left[\mathbf{k}_{N\times 1}(x) \odot \mathbf{u}_{N\times 1}\right]^\top \mathbf{w}_{N\times 1}.
\end{equation}
Here, $\mathbf{k}_j(x) = \mathcal{K}(x, \mathbf{r}_j)$, $\mathbf{u}_j = u(\mathbf{r}_j)$, and $N$ represents the number of nodal points. By computing the matrix $\mathbf{K}$ as $\mathbf{K}_{i,j} = \mathcal{K}(\mathbf{x}_i, \mathbf{r}_j)$, one can compute the integral operator via a matrix-vector product:
\begin{equation}
    \mathbf{I}_{N\times 1} = \left[\mathbf{K}_{N\times N}^\top \odot \mathbf{u}_{N\times 1}\right]_{N\times N}^\top \cdot \mathbf{w}_{N\times 1},
\end{equation}
where $\cdot$ denotes a matrix-vector product, and $\odot$ denotes the Hadamard product with broadcasting.

For integrals over a finite domain $[a,b]$ that differs from $[-1,1]$, a simple affine mapping can be applied to the roots as $\tilde{\mathbf{r}}_i = \left[(b - a) \mathbf{r} + (a + b)\right]/2$, with a similar transformation applied to the integration result:
\begin{equation}
    {\mathbf{I}} = \frac{b-a}{2}\odot \left[\mathbf{K}^\top \odot \mathbf{u}\right]^\top \cdot \mathbf{w}.
\end{equation}
This implies that the Fredholm integral operator can be approximated in $O(N^2)$ time, potentially accelerated by Single Instruction Multiple Data (SIMD) techniques.

\subsection{Volterra Integral Operator}
In Volterra integral operators, the bounds of integration are functions of the independent variable $x$. These operators are mathematically represented as:
\begin{equation}
    \mathcal{I}(u)(x) = \int_{g(x)}^{h(x)} \mathcal{K}(x,t) u(t) \, dt.
    \label{eq:volterra_operator}
\end{equation}
Unlike Fredholm integral equations, computing Volterra integrals presents additional challenges because the integration range varies with each training point. Specifically, for each $x_i$, the integral is given by:
\begin{equation}
    I(x_i) = \frac{h(\mathbf{r}_i) - g(\mathbf{r}_i)}{2} \odot \left[\mathbf{k}(x_i) \odot \mathbf{u}^{(i)}\right]^\top \cdot \mathbf{w},
\end{equation}
where $\mathbf{k}_j(x_i) = \mathcal{K}(x_i, \mathbf{r}_j)$, $\mathbf{w}$ represents the quadrature weights,  $\mathbf{r}$ is a vector containing the roots of the orthogonal polynomial, and $\mathbf{u}^{(i)}$ is evaluated to perform the quadrature over the interval $[g(x_i), h(x_i)]$:
\begin{equation*}
    \mathbf{u}^{(i)}_j = u(\frac{h(\mathbf{r}_i) - g(\mathbf{r}_i)}{2} \odot \mathbf{r}_j + \frac{h(\mathbf{r}_i) + g(\mathbf{r}_i)}{2}).
\end{equation*}
To compute the vectorized form of $\mathbf{I}_i = I(x_i)$, we need to evaluate $\mathbf{u}^{(i)}$ by calculating $u(t)$ at the following set of points, arranged in a matrix:
{\footnotesize
\begin{equation*}
    \mathbf{R} = \begin{bmatrix}
\frac{h(\mathbf{r}_1) - g(\mathbf{r}_1)}{2} \odot \mathbf{r}_1 + \frac{h(\mathbf{r}_1) + g(\mathbf{r}_1)}{2}
&
\cdots
&
\frac{h(\mathbf{r}_1) - g(\mathbf{r}_1)}{2} \odot \mathbf{r}_N + \frac{h(\mathbf{r}_1) + g(\mathbf{r}_1)}{2}
\\
&\vdots \\
&
\frac{h(\mathbf{r}_i) - g(\mathbf{r}_i)}{2} \odot \mathbf{r}_j + \frac{h(\mathbf{r}_i) + g(\mathbf{r}_i)}{2}
&
\\
&\vdots 
\\

\frac{h(\mathbf{r}_N) - g(\mathbf{r}_N)}{2} \odot \mathbf{r}_1 + \frac{h(\mathbf{r}_N) + g(\mathbf{r}_N)}{2}
&
\cdots
&
\frac{h(\mathbf{r}_N) - g(\mathbf{r}_N)}{2} \odot \mathbf{r}_N + \frac{h(\mathbf{r}_N) + g(\mathbf{r}_N)}{2}
\end{bmatrix}_{N \times N}.
\end{equation*}
}
This method requires computing the network output at $N^2$ distinct nodes. These nodes may differ from the training points as they increase quadratically with $N$, which can slow down the forward phase.

A similar method applies to evaluating the kernel function $\mathcal{K}(x,t)$, which must be computed at these points for variable $t$ and for each $x_i = \mathbf{r}_i$, i.e., $\mathbf{K}_{i,j} = \mathcal{K}(\mathbf{x}_i, \mathbf{R}_{i,j})$. As with the Fredholm integral operator, the kernel matrix can be precomputed before the training phase to accelerate the learning process. However, the network output vector $\mathbf{u}_{N^2\times 1}$ must be computed during the training phase. Combining everything, the vector $\mathbf{I}$ can be calculated as:
\begin{equation*}
    \mathbf{I}_{N\times 1} = \frac{h(\mathbf{r}) - g(\mathbf{r})}{2} \odot \left[\mathbf{K}_{N\times N} \odot \mathbf{u}_{N\times N}\right] \cdot \mathbf{w}_{N\times 1},
\end{equation*}
where $\mathbf{u}_{N\times N}$ is the reshaped form of the network output. 

\subsection{Multi-dimensional Integral Operators}
Many practical applications of integral equations involve multi-dimensional integral operators \cite{jiang2022review}. The process of approximating these operators is similar to the one-dimensional case but requires multiple iterations. Consider the following two-dimensional Fredholm integral operator as an example:
\begin{equation*}
    \mathcal{I}(u)(x, y) = \int_a^b \int_c^d \mathcal{K}(x,y,s,t) u(s,t) \, dt \, ds.
\end{equation*}
To approximate the unknown function $u(x,y)$, the network must be trained on a mesh grid composed of $N_x \times N_y$ data points. These points are typically chosen as the roots of orthogonal polynomials, denoted by $\mathbf{r}^{(x)}$ and $\mathbf{r}^{(y)}$, corresponding to the variables $x$ and $y$, respectively. The kernel matrix $\mathbf{K}$ is then computed as a four-dimensional tensor of shape $\mathbf{K} \in \mathbb{R}^{N_x \times N_y \times N_x \times N_y}$, with elements defined as $\mathbf{K}_{i,j,k,l} = \mathcal{K}(\mathbf{x}_i, \mathbf{y}_j, \mathbf{r}^{(x)}_k, \mathbf{r}^{(y)}_l)$. The Fredholm operator can then be computed as follows:
\begin{align*}
    &\tilde{\mathbf{I}}_{N_x \times N_y \times N_x} = \frac{d-c}{2} \odot \left[\mathbf{K}_{N_x \times N_y \times N_x \times N_y} \odot \mathbf{u}_{N_x \times N_y}\right]_{N_x \times N_y \times N_x \times N_y} \cdot \mathbf{w}_{N_y \times 1},\\
    &\mathbf{I}_{N_x \times N_y} = \frac{b-a}{2} \odot \tilde{\mathbf{I}}_{N_x \times N_y \times N_x} \cdot \mathbf{w}_{N_x \times 1},
\end{align*}
where $\mathbf{u}_{N_x \times N_y}$ is the reshaped form of the neural network prediction on the mesh grid of $\mathbf{x}$ and $\mathbf{y}$. For three dimensions, a similar approach can be applied to compute the six-dimensional tensor $\mathbf{K}$, followed by the computation of $\mathbf{I}_{N_x \times N_y \times N_z}$:
\begin{align*}
    &\hat{\mathbf{I}}_{N_x \times N_y \times N_z \times N_x \times N_y} = \frac{f-e}{2} \odot \left[\mathbf{K}_{N_x \times N_y \times N_z \times N_x \times N_y \times N_z} \odot \mathbf{u}_{N_x \times N_y \times N_z}\right] \cdot \mathbf{w}_{N_z \times 1},\\
    &\tilde{\mathbf{I}}_{N_x \times N_y \times N_z \times N_x} = \frac{d-c}{2} \odot \hat{\mathbf{I}}_{N_x \times N_y \times N_z \times N_x \times N_y} \cdot \mathbf{w}_{N_y \times 1},\\
    &{\mathbf{I}}_{N_x \times N_y \times N_z} = \frac{b-a}{2} \odot \tilde{\mathbf{I}}_{N_x \times N_y \times N_z \times N_x} \cdot \mathbf{w}_{N_x \times 1}.
\end{align*}

The formulation of the Volterra integral operator needs some adjustments. Consider a two-dimensional Volterra integral operator of the form:
\begin{equation*}
    \mathcal{I}(u)(x, y) = \int_{g_1(x)}^{h_1(x)} \int_{g_2(y)}^{h_2(y)} \mathcal{K}(x, y, s, t) \, u(s, t) \, dt \, ds.
\end{equation*}
Similar to the multi-dimensional Fredholm operator, the network should be trained on a grid of $N_x \times N_y$ nodes, where these nodes are selected as the roots of orthogonal polynomials. To compute the kernel function in this context, we first need to compute two matrices, $\mathbf{R}^{(x)} \in \mathbb{R}^{N_x \times N_x}$ and $\mathbf{R}^{(y)} \in \mathbb{R}^{N_y \times N_y}$, as follows:
\begin{equation*}
    \begin{aligned}
        \mathbf{R}^{(x)}_{i,j} &= \frac{h_1(\mathbf{r}_i^{(x)}) - g_1(\mathbf{r}_i^{(x)})}{2} \odot \mathbf{r}_j^{(x)} + \frac{h_1(\mathbf{r}_i^{(x)}) + g_1(\mathbf{r}_i^{(x)})}{2}, \\
        \mathbf{R}^{(y)}_{i,j} &= \frac{h_2(\mathbf{r}_i^{(y)}) - g_2(\mathbf{r}_i^{(y)})}{2} \odot \mathbf{r}_j^{(y)} + \frac{h_2(\mathbf{r}_i^{(y)}) + g_2(\mathbf{r}_i^{(y)})}{2}.
    \end{aligned}
\end{equation*}
These matrices are utilized in two steps: first for computing the kernel function, and second for computing the tensor $\mathbf{u}$ within the integral operator:
\begin{equation*}
\begin{aligned}
    \mathbf{K}_{N_x \times N_y \times N_x \times N_y} &= \mathcal{K}(\mathbf{r}^{(x)}, \mathbf{r}^{(y)}, \mathbf{R}^{(x)}, \mathbf{R}^{(y)}),\\
    \mathbf{u}_{N_x \times N_y \times N_x \times N_y} &= u(\hat{\mathbf{R}}),
\end{aligned}
\end{equation*}
where $\hat{\mathbf{R}} \in \mathbb{R}^{N_x^2\times N_y^2 \times 2}$ is a reshaped mesh grid formed from $\mathbf{R}^{(x)}$ and $\mathbf{R}^{(y)}$. Using these tensors, the integral operator can be approximated as:
\begin{equation*}
    \mathbf{I}_{N_x \times N_y} =
    \frac{h_1(\mathbf{r}^{(x)}) - g_1(\mathbf{r}^{(x)})}{2} \odot
    \left[\frac{h_2(\mathbf{r}^{(y)}) - g_2(\mathbf{r}^{(y)})}{2} \odot
    \left[\mathbf{K} \odot \mathbf{u}\right] \cdot \mathbf{w}\right] \cdot \mathbf{w}.
\end{equation*}
In this case, the time complexity of the computation is \( O(N_x^2 \times N_y^2) \), which is manageable for a two-dimensional operator. By applying a similar procedure, one can compute the approximation for a three-dimensional Volterra operator. 

\subsection{Fractional Operators}
The concept of fractional derivatives can be traced back to the 17th century, with mathematicians such as Leibniz and Euler being among the first to introduce the idea \cite{loverro2004fractional}. In recent years, this field has gained considerable attention due to its wide range of applications across various domains, including physics, engineering, finance, and control theory \cite{sun2018new, dalir2010applications, arora2022applications, yang2022fractional}.

A key feature of fractional calculus is the diversity of its definitions, which can vary depending on the specific application or problem. The Riemann–Liouville, Hadamard, and Atangana–Baleanu integrals are among the most well-known fractional integrals, each providing a different approach to integration. Similarly, fractional derivatives have multiple definitions, such as those proposed by Riemann–Liouville, Caputo, Caputo–Fabrizio, Atangana–Baleanu, and Grünwald–Letnikov \cite{valerio2013fractional}.

Typically, fractional derivatives and their definitions are expressed in terms of integrals. A notable example is the Caputo fractional derivative of a function \( u(x) \) of order \( \alpha \in \mathbb{R}^+ \), which is defined as:
\begin{equation}
    _0^{C}\mathfrak{D}^{\alpha}_{t} u(x) = \frac{d^\alpha u(x)}{dx^\alpha} = \frac{1}{\Gamma(v-\alpha)} \int_{0}^{x} (x-s)^{v-\alpha-1} \frac{d^v u(s)}{ds^v} \, ds,
\end{equation}
where \( v-1 < \alpha < v \) for \( v \in \mathbb{Z}^+ \), and \( \Gamma(z) = \int_{0}^{\infty} t^{z-1} e^{-t} \, dt \) denotes the Gamma function. It has been shown that as \( \alpha \) approaches an integer, this definition converges to the \( v \)-th classical derivative \cite{valerio2013fractional, podlubny1998fractional}. This definition is characterized by useful properties such as linearity and additivity. One valuable property is \cite{taheri2024accelerating}:
\begin{equation}
    _0^{C}\mathfrak{D}^{p}_{t} u(x) = _0^{C}\mathfrak{D}^{v}_{t} [_0^{C}\mathfrak{D}^{\alpha}_{t} u(x)] = \frac{d^v}{dx^v} [_0^{C}\mathfrak{D}^{\alpha}_{t} u(x)] = _0^{C}\mathfrak{D}^{\alpha}_{t} \frac{d^v}{dx^v} u(x), 
    \label{eq:property}
\end{equation}
where \( p = v + \alpha \), \( v \) is the integer part of \( p \), and \( 0 < \alpha < 1 \). This property ensures that a fractional derivative of order \( p \) can be obtained by computing the \( v \)-th derivative of \( _0^{C}\mathfrak{D}^{\alpha} u(x) \).

Calculating this definition can be challenging because of the integral term, which involves a singularity. Although one might consider using techniques similar to those applied for approximating the Volterra integral operator, the singularity requires evaluating the approximation at numerous points within the integration domain. This process can become computationally expensive, particularly in the context of deep neural networks. To address this, the following theorem presents a mathematical discretization technique that strikes a balance between accuracy and computational efficiency, simplifying the problem to matrix-vector multiplication \cite{taheri2024accelerating}.

\begin{theorem}
    Let $0 < \alpha < 1$ and the interval $[0, x]$ is discretized to $n+1$ points, $0 = x_0 < x_1 < \dots < x_n=x$. Then the following linear combination approximates the Caputo fractional derivative of order $\alpha$:
\begin{equation}
        {}_{0}^{C}\mathfrak{D}_{x_n}^{\alpha}u{(x)} = \mathbf{u}^\top \boldsymbol{\nu} = \sum_{k=0}^n\nu_k u(x_k),
    \end{equation}
where $u(\cdot)$ is the desired function and $\nu_k$ are real-valued weights.
    \label{thm:L1}
\end{theorem}
\begin{proof}
We begin by dividing the integration into $n$ non-equidistant intervals. 
  \begin{equation*}
    \begin{aligned}
{}_{0}^{C}\mathfrak{D}_{x_n}^{\alpha}u{(x)} &=\frac{1}{\Gamma{(1-\alpha)}}\int_{0}^{x_{n}}\frac{u^{\prime}(x)}{(x_{n}-x)^{\alpha}}dx  \\
&=\frac1{\Gamma(1-\alpha)}\sum_{k=0}^{n-1}\int_{x_k}^{x_{k+1}}\frac1{(x_n-x)^\alpha}u^{\prime}(x)dx.
\end{aligned}
\end{equation*}
Within each interval, the derivative $u'(x)$ can be approximated using the forward finite difference method:
\begin{equation*}
\begin{aligned}
{}_{0}^{C}\mathfrak{D}_{x_n}^{\alpha}u{(x)} & \approx\frac1{\Gamma(1-\alpha)}\sum_{k=0}^{n-1}\int_{x_k}^{x_{k+1}}\frac{1}{(x_n-x)^k}\frac{u(x_{k+1})-u(x_k)}{x_{k+1}-x_k}dx\\
& \approx\frac1{\Gamma(1-\alpha)}\sum_{k=0}^{n-1}\frac{u(x_{k+1})-u(x_k)}{x_{k+1}-x_k}\int_{x_k}^{x_{k+1}}\frac{dx}{(x_n-x)^\alpha} .\\
\end{aligned}
\end{equation*}
Subsequently, analytical integration is applied to simplify the expression:
   \begin{equation*}
    \begin{aligned}
{}_{0}^{C}\mathfrak{D}_{x_n}^{\alpha}u{(x)} & \approx \frac1{\Gamma(1-\alpha)}\sum_{k=0}^{n-1}\frac{u(x_{k+1})-u(x_{k+1})}{x_{k+1}-x_{k}}\times\left[-\frac{(x_{n}-x_{k})^{1-\alpha}-(x_{n}-x_{k+1})^{1-\alpha}}{(\alpha-1)}\right] \\
&\approx\frac1{\Gamma(2-\alpha)}\sum_{k=0}^{n-1}\left[\frac{(x_n-x_k)^{1-\alpha}-(x_{n}-x_{k+1})^{1-\alpha}}{x_{k+1}-x_k}\right]\left[u(x_{k})-u(x_{k+1})\right].
    \end{aligned}
\end{equation*}
Introducing $\mu_k$ as the weight component in the summation, a straightforward reformulation leads to:
\begin{equation*}
    \begin{aligned}
        {}_{0}^{C}\mathfrak{D}_{x_n}^{\alpha}u{(x)} &\approx\frac{1}{\Gamma{(2-\alpha)}}\sum_{k=0}^{n-1}\mu_{k}\left[u{(x_k)}-u{(x_{k+1})}\right] \\
&\approx\frac{1}{\Gamma(2-\alpha)}\sum_{k=0}^{n-1} \left[ \mu_{k}-\mu_{k-1}\right] u(x_{k}) \\
&\approx\sum_{k=0}^n\nu_k u(x_k), 
    \end{aligned}
\end{equation*}
where $\nu_k = {\left[\mu_k - \mu_{k-1}\right]}/{\Gamma(2-\alpha)}$ and:
\begin{equation*}
    \mu_k = \begin{cases}\displaystyle\frac{(x_n-x_k)^{1-\alpha}-(x_{n}-x_{k+1})^{1-\alpha}}{x_{k+1}-x_k} & 0\le k < n\\ 0 & \text{otherwise}.\end{cases}
\end{equation*}
\end{proof}

\begin{corollary}
To compute the Caputo fractional derivative for \(\alpha > 1\), one can first determine the integer-order derivative and then apply the fractional derivative of order \(\alpha\), where \(\alpha\) represents the fractional part of the total derivative order (see equation \ref{eq:property}). The integer-order derivative can be efficiently computed using automatic differentiation, ensuring that the accuracy is not compromised.
\end{corollary}
\begin{theorem}
\label{thm:matrix}
 The operational matrix of the Caputo fractional derivative can be obtained using the lower triangular matrix $\mathcal{M}$:
\begin{equation*}
    \mathcal{M} = \begin{bmatrix}
0 \\
\nu_0^{(1)} & \nu_1^{(1)}\\
\nu_0^{(2)} & \nu_1^{(2)} & \nu_2^{(2)}\\
&\vdots\\
\nu_0^{(N-2)} & \nu_1^{(N-2)} & \nu_2^{(N-2)} & \nu_3^{(N-2)} & \cdots & \nu_{N-2}^{(N-2)}\\
\nu_0^{(N-1)} & \nu_1^{(N-1)} & \nu_2^{(N-1)} & \nu_3^{(N-1)} & \cdots & \nu_{N-2}^{(N-1)} & \nu_{N-1}^{(N-1)}
\end{bmatrix}.
\end{equation*}
Therefore, for the vector-valued function $\mathbf{u}_i = u(x_i)$ consisting of $N$ elements and arbitrary nodes $x_i$, the Caputo fractional derivative of order $\alpha$ can be efficiently computed using the operational matrix of the derivative:
\begin{equation*}
    _{}^{C}\mathbf{u}^{(\alpha)} = \frac{d^\alpha \mathbf{u}}{dx^\alpha} \approx \mathcal{M} \mathbf{u}.
\end{equation*}
\end{theorem}
\begin{proof}
The proof can be finalized by approximating the $i^\text{th}$ element of $_{}^{C}\mathbf{u}^{(\alpha)}$ using Theorem \ref{thm:L1}:
    \begin{equation*}
        _{}^{C}\mathbf{u}^{(\alpha)}_i = \sum_{k=0}^i\nu_k^{(i)} \mathbf{u}_{i,k},
    \end{equation*}
where $\nu_k^{(i)} = {\left[\mu_k^{(i)} - \mu_{k-1}^{(i)}\right]}/{\Gamma(2-\alpha)}$ and:
\begin{equation*}
    \mu_k^{(i)} = \begin{cases}\displaystyle\frac{(x_i-x_k)^{1-\alpha}-(x_i-x_{k+1})^{1-\alpha}}{x_{k+1}-x_k} & 0\le k < N\\ 0 & \text{otherwise}.\end{cases}
\end{equation*}
\end{proof}

\section{Numerical Results}
In this section, we validate the proposed method by applying it to the numerical solution of various mathematical problems involving integral operators. We begin with a sensitivity analysis on the quadrature method to explore the hyperparameter space of the problem. Following this, we test different types of integral equations, including one-dimensional and multi-dimensional systems, as well as integro-differential equations involving ordinary, partial, and fractional derivatives. Next, we address optimal control problems characterized by fractional derivatives, delay terms, integro-differential constraints, and multi-dimensional cases. We also tackle inverse problems that include integral terms. To demonstrate the accuracy of our approach, we consider various integral operators such as Fredholm, Volterra, and Volterra-Fredholm, exploring different configurations, including types, linearities, singularities, and analytical dynamics. Most examples are drawn from Wazwaz's authoritative work on linear and nonlinear integral equations \cite{wazwaz2011linear}, unless otherwise cited. In either case, it is clear that the source terms of the equations align with the provided exact solutions.

For the majority of experiments, we selected benchmark integral equations with known analytical solutions, which allow for precise comparisons. We report the Mean Absolute Error (MAE) between the exact solution and the solution obtained using our Physics-Informed Neural Network framework:
\begin{equation}
    \mathrm{MAE}(u, \hat{u}) = \sum_{i=1}^N |u(x_i) - \hat{u}(x_i)|,
\end{equation}
where \(N\) represents the number of test points, \(u(\cdot)\) denotes the exact solution, and \(\hat{u}(\cdot)\) is the solution predicted by the network.

The proposed method was implemented in Python 3.12 using the PyTorch framework (Version 2.2.2), which supports automatic differentiation on the computational graph. All experiments were conducted on a personal computer equipped with an Intel Core i5-1235U CPU and 24GB of RAM, running the EndeavourOS Linux distribution.

\subsection{Sensitivity Analysis and Hyperparameters}
Before addressing the numerical solution of integral equations, we first compare the integration techniques mentioned and demonstrate how the choice of algorithm can impact solution accuracy. To achieve this, we chose four distinct functions over specified domains, divided the domain $[a, b]$ into $N$ nodes, and utilized various numerical integration methods: Gaussian quadrature, Monte Carlo integration, Newton-Cotes with $N$ nodes, and the basic Newton-Cotes method using just two nodes and weights, commonly known as the trapezoid rule. This method for a set of $N$ function evaluations can be expressed as:
\begin{equation*}
    \mathrm{trapz}(\mathbf{u}) = \frac{\delta}{2} \sum_{i=1}^{N-1} \mathbf{u}_i + \mathbf{u}_{i-1},
\end{equation*}
where $\delta$ represents the distance between two nodes $x_i$ and $x_{i+1}$, and $\mathbf{u}_i = u(x_i)$. The comparison is illustrated in Figure \ref{fig:compare-quad}. In this figure, we used a 32-bit hardware floating point with 7-decimal-point precision for storing weights and function values. This limitation was chosen because the PyTorch framework, which will be utilized to define neural networks and implement automatic differentiation, operates under this hardware specification.

As shown in the figure, the Monte Carlo method exhibits the lowest accuracy among the methods tested, as anticipated, due to its reliance on a large number of function evaluations. Following this, the Trapezoid method shows accuracy similar to Monte Carlo, reflecting its simplicity as a basic quadrature technique. The standard Newton-Cotes method with $N$ nodes achieves better accuracy than the Trapezoid method, as expected. However, this method also displays increasing numerical instability with more nodes, likely due to Runge's phenomenon. In contrast, the Gauss-Legendre quadrature method consistently demonstrates the highest stability and accuracy, even with a relatively small number of function evaluations. In some cases, it reaches the best achievable accuracy within the limits of 32-bit floating point precision, making it an excellent choice for subsequent simulations of integral equations.
\begin{figure}[ht]
    \centering
    \includegraphics[width=0.95\textwidth]{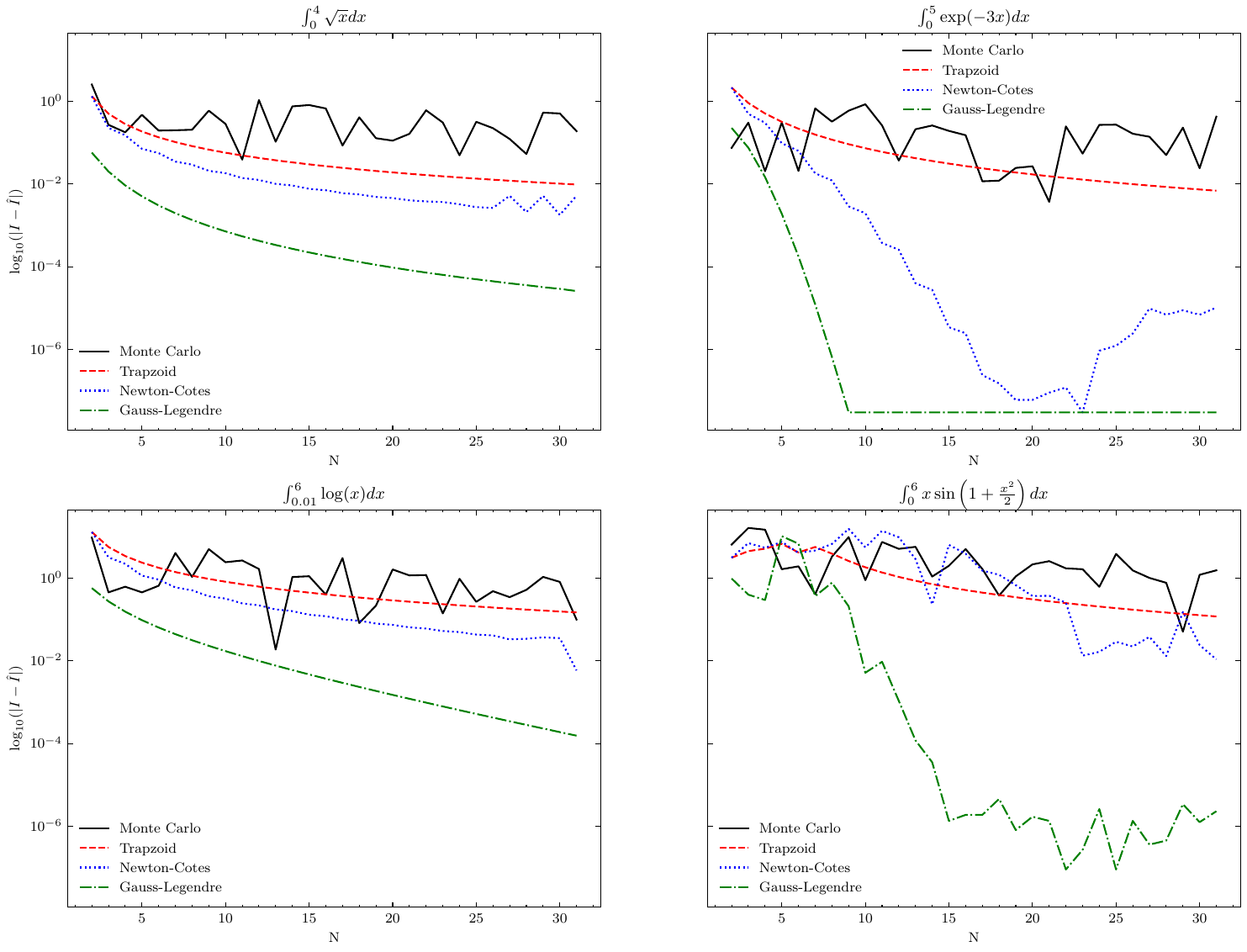}
    \caption{A comparison between the Monte Carlo, Newton-Cotes, and Gaussian Quadrature methods for approximating the integral of various functions reveals that Gaussian Quadrature provides greater accuracy with reduced numerical instability. The CPU times for the Monte Carlo, Trapezoid method, Newton-Cotes, and Gauss-Legendre methods with $N=30$ nodal points are $9.5 \pm 0.37$, $19 \pm 0.1$, $8.13 \pm 0.09$, and $8.3 \pm 0.03$ microseconds, respectively. In contrast, computing the derivative of the same function using automatic differentiation takes $21 \pm 0.89$ microseconds. It is evident that the Gauss-Legendre approach is both accurate and efficient.}
    \label{fig:compare-quad}
\end{figure}

In the next experiment, we conduct a hyperparameter analysis on four different integral equations with distinct exact solutions. We consider two Fredholm and two Volterra integral equations with identical structures, except for the exact solutions and the corresponding source terms. The equations are as follows:
\begin{align}
    u(x) &= \exp(x) + x - \frac{4}{3} + \int_0^1 t \, u(t) \, dt \tag{EX.1}, \\
    u(x) &= \sin(2x) + \frac{\pi}{2} + \int_0^{\frac{\pi}{2}} t \, u(t) \, dt \tag{EX.2}, \\
    u(x) &= x + 2 \exp(x) - 1 - \frac{x^3}{3} - x \exp(x) + \int_0^x t \, u(t) \, dt \tag{EX.3}, \\
    u(x) &= \frac{3 \sin(2x)}{4} + \frac{x \cos(2x)}{2} + \int_0^x t \, u(t) \, dt \tag{EX.4},
\end{align}
where for \textbf{(EX.1)} and \textbf{(EX.3)} the exact solution is $u(x) = \exp(x)$ over $\Delta = [0, 1]$, and for \textbf{(EX.2)} and \textbf{(EX.4)} the dynamics follow $u(x) = \sin(2x)$ over $\Delta = [0, \pi]$. We compared the CPU time and MAE of these four examples across different sets of hyperparameter configurations, including varying numbers of training points, hidden layers, and neurons in hidden layers, to assess how these parameters impact the accuracy of the learned solution. Figure \ref{fig:hp-space} presents these experiments.

In the first experiment, we selected a neural network with two hidden layers and an architecture of $[1, 10, 10, 1]$, where $10$ represents the number of neurons in each layer, and we examined different numbers of training points (or Gaussian nodes). As shown, increasing the number of training points up to 10 is sufficient for accurate network prediction; beyond this, the accuracy remains approximately constant. The accuracy for examples \textbf{(EX.2)} and \textbf{(EX.4)} is lower than for the other two, due to their stiff solutions and wider problem domains. In both cases, the Fredholm integral equations are more accurately solved compared to the Volterra integral equations. Additionally, the time complexity for solving Fredholm equations is significantly lower than for Volterra equations. This is directly related to the need to predict the function at a set of different nodes, which grows quadratically with the number of training data points when computing the integral term in Volterra equations.

The next experiment involved setting the number of training points to 50 and evaluating the network's accuracy as its depth increased. We fixed the number of hidden neurons per layer to 10 and tested the proposed method on networks with hidden layers ranging from 2 to 10. We observed that as the network depth increased, its accuracy decreased, likely due to the vanishing gradient problem. The CPU time for training also increased for Volterra-type problems, while it had a lower impact on Fredholm equations.

In the following experiment, we fixed the number of layers to two, meaning an architecture of $[1, H, H, 1]$, and evaluated different values for $H$, the number of hidden neurons. We found that fewer than 5 hidden neurons significantly reduced network accuracy, while the number of neurons had little impact on CPU time due to the vectorized formulation.

Based on these experiments, we will use the architecture $[d, 10, 10, 1]$ for all subsequent problems, where $d$ is the dimensionality of the problem. We will employ a hyperbolic tangent as the activation function, the L-BFGS algorithm for optimizing the network weights, with a learning rate ranging from $0.01$ to $0.1$ and a number of iterations from 50 to 250.

\begin{figure}[ht]
    \centering
    \includegraphics[width=0.95\textwidth]{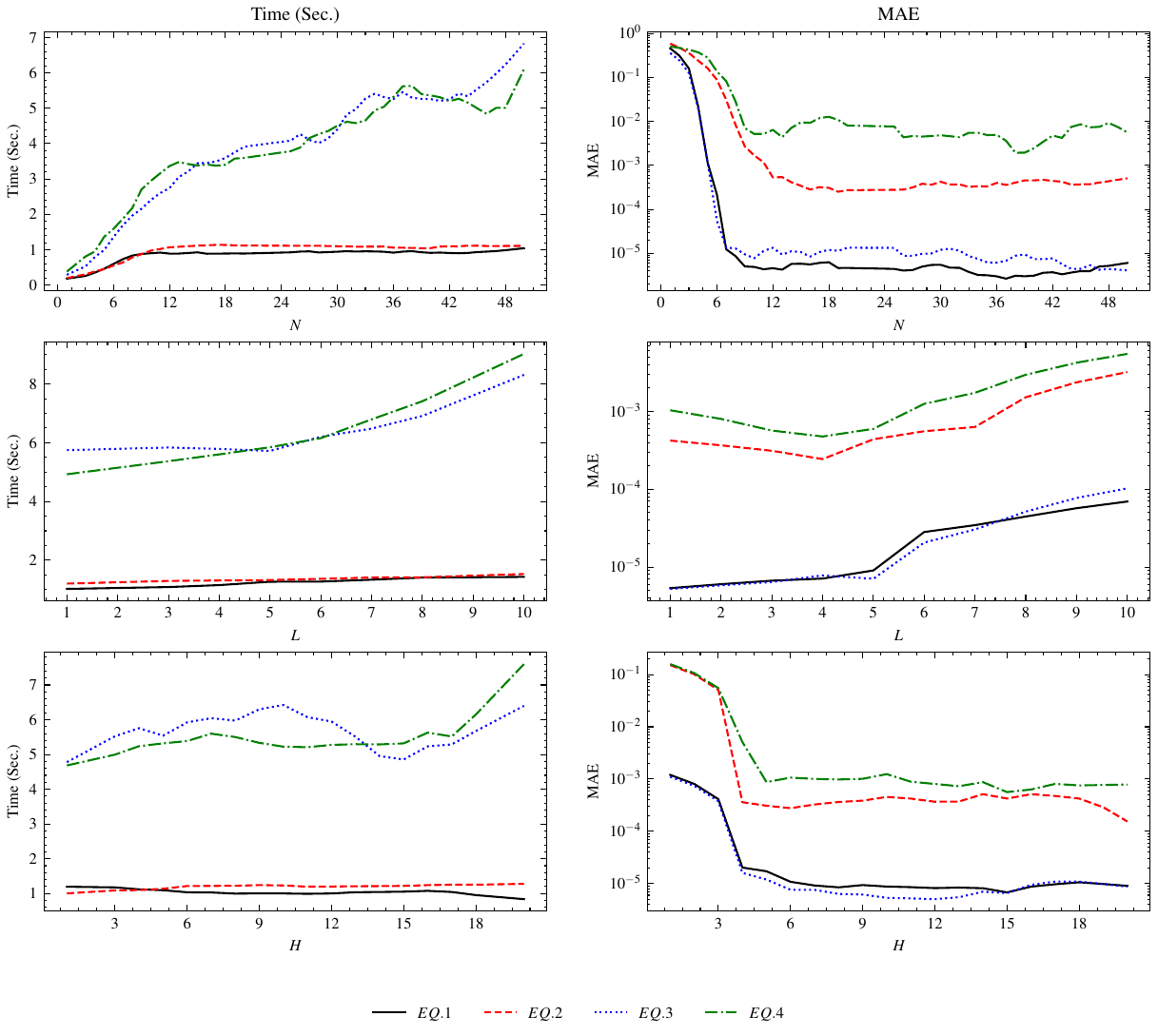}
    \caption{A comparison of four different Volterra and Fredholm integral equations: two with exponential solutions and two with stiff dynamics. The first row illustrates the impact of the number of training points on network accuracy. The middle row examines the effect of varying the number of layers, while the bottom row explores different numbers of hidden neurons in a two-layer neural network. The left column compares CPU time, while the right column evaluates MAE on test data. The reported times represent the mean of five separate runs of the algorithm, using a learning rate of $0.1$ and $50$ epochs. In all cases, we use the hyperbolic tangent function as the nonlinearity mapping.}
    \label{fig:hp-space}
\end{figure}

\subsection{One-dimensional Integral Equations}
In this section, we consider one-dimensional integral equations of the form:
\begin{equation*}
    \kappa u(x) = \mathcal{S}(x) + \int_\Delta \mathcal{K}(x,t) \zeta(u(t)) \, dt,
\end{equation*}
where $\mathcal{S}(x)$ represents the source term, $\mathcal{K}(x,t)$ is the kernel of the integral operator, and $\zeta(x)$ is a function indicating the linearity of the problem. The parameter $\kappa \in \mathbb{R}$, a known constant, determines the type of problem: if $\kappa = 0$, the problem is classified as a first-kind IE; otherwise, it is a second-kind IE. The problem domain, denoted as $\Delta = [a,b]$, will be applied to all types of equations. Specifically, we assume the domain for the Fredholm operator to be $a,b = 0,1$, and for the Volterra operator, we set $g(x), h(x) = 0, x$ for $x \in \Delta$. For the Volterra-Fredholm problem, the following equation is considered:
\begin{equation*}
    \kappa u(x) = \mathcal{S}(x) + \int_a^b \mathcal{K}_f(x,t) \zeta(u(t)) \, dt + \int_0^x \mathcal{K}_v(x,t) \zeta(u(t)) \, dt.
\end{equation*}

Using the method proposed in Section \ref{sec:3}, we simulate a variety of integral equations and present the problem specifications along with the simulation MAE results in Table \ref{tbl:1d}. It is observed that for most problems, except for the Abel-type singular integral equations, the results are highly accurate. For Abel problems, increasing the number of nodal points or employing a different Gaussian quadrature method, other than Gauss-Legendre, may improve accuracy.

\begin{table}[ht]
\begin{tabular}{@{}llllllll@{}}
\toprule
Type & $\Delta$& $\zeta(x)$ & $\kappa$ & Exact & Source term & Kernel & MAE \\ \midrule
Fredholm & $[0,1]$ & $x$ & $1$ & $x+e^x$ & $e^x + x - \frac{4}{3}$ & $t$ & $4.45 \times 10^{-5}$ \\
Fredholm & $[0,1]$ & $x$ & $1$ & $x+e^x$ & $e^x + \frac{x}{2} - \frac{4}{3} + x e$ & $t-x$ & $1.56 \times 10^{-5}$ \\
Fredholm & $[0,1]$ & $e^x$ & $1$ & $x$ & $x e$ & $-x$ & $1.48 \times 10^{-6}$ \\
Volterra & $[0,1]$ & $x$ & $1$ & $x+e^x$ & $2e^x - 1 + \frac{x^3}{6}$ & $t-x$ & $2.34 \times 10^{-5}$ \\
Volterra & $[0,1]$ & $x$ & $0$ & $\sin(x)$ & $\sin(x) - x\cos(x)$ & $-t$ & $7.85 \times 10^{-4}$ \\
Volterra & $[0,1]$ & $x^2$ & $0$ & $e^x$ & $e^{2x} - e^{x}$ & $-e^{x-t}$ & $3.29 \times 10^{-4}$ \\
Volterra & $[0,1]$ & $x$ & $1$ & $e^x$ & $1$ & $1$ & $1.11 \times 10^{-5}$ \\
Volterra & $[0,1]$ & $x^2$ & $1$ & $e^x$ & $e^x - \frac{1}{2}\left(e^{2x} - 1\right)$ & $1$ & $2.30 \times 10^{-5}$ \\
Volterra-Fredholm & $[0,1]$ & $x$ & $1$ & $x+e^x$ & $2e^x - \frac{x}{2} - \frac{7}{3} + \frac{x^3}{6} + x e$ & $\mathcal{K}_f = \mathcal{K}_v = t-x$ & $1.95 \times 10^{-5}$ \\
Volterra-Fredholm & $[0,1]$ & $x$ & $1$ & $xe^x$ & $e^x - 1 - x$ & $\mathcal{K}_f = x, \mathcal{K}_v = 1$ & $3.41 \times 10^{-5}$ \\
Abel & $[0,1]$ & $x$ & $0$ & $x$ & $\frac{4}{3}x^{\frac{3}{2}}$ & $\frac{-1}{\sqrt{x-t}}$ & $3.27 \times 10^{-3}$ \\
Abel & $[0,1]$ & $x^3$ & $0$ & $x$ & $\frac{32}{35}x^{\frac{7}{2}}$ & $\frac{-1}{\sqrt{x-t}}$ & $1.58 \times 10^{-3}$ \\
Fredholm & $[0,\infty)$ & $x$ & $1$ & $2e^{-x}$ & $e^{-x}$ & $e^{-(x+t)}$ & $3.71 \times 10^{-5}$ \\ \bottomrule
\end{tabular}
\caption{Examples of one-dimensional Fredholm, Volterra, and Abel-type integral equations over finite and semi-infinite domains, along with the corresponding MAE of their neural network solutions.}
\label{tbl:1d}
\end{table}

\subsection{Integro-differntial Integral Equations}
In this section, we validate the proposed method for integral equations that involve differentiation operators. We examine three types of differentiation: ordinary, partial, and fractional. For ordinary and partial derivatives, we use automatic differentiation to compute the derivatives of the unknown solution with respect to the inputs. For fractional derivatives, we apply the operational matrix of fractional differentiation as described in Theorem \ref{thm:matrix}.

\subsubsection{Ordinal Integro-differential Equations}
As the first experiment, consider the following form of an ordinal integro-differential equation:
\begin{equation*}
    \kappa \frac{d^v}{dx^v}u(x) = \mathcal{S}(x) + \int_\Delta \mathcal{K}(x,t) \zeta(u(t)) \, dt,
\end{equation*}
where $v \in \mathbb{Z}^+$ denotes the differentiation order. For all configurations of this problem, we apply two boundary conditions with known data from the exact solution. Other configurations will follow the approach used in the previous section on one-dimensional integral equations. Table \ref{tbl:oide} presents various cases of this problem simulated using our method, along with the corresponding absolute error. The results demonstrate that the proposed method is highly accurate, even for cases where the derivatives of the unknown solution appear under the integral sign.

\begin{table}[ht]
\begin{tabular}{@{}llllllll@{}}
\toprule
Type & $\zeta(x)$ & $v$ & $\kappa$ & Exact & Source term & Kernel & MAE \\ \midrule
Fredholm & $x$ & $2$ & $1$ & $e^x$ & $1-e+e^x$ & $1$ & $2.74 \times 10^{-7}$ \\
Fredholm & $x$ & $2$ & $1$ & $\sin(x)$ & $\cos(x) - 1 + \cos(1)$ & $1$ & $1.03 \times 10^{-6}$ \\
Fredholm & $x$ & $2$ & $1$ & $e^x + x$ & $\frac{1}{2} - e + e^x$ & $1$ & $3.19 \times 10^{-6}$ \\
Fredholm & $x^2$ & $1$ & $1$ & $x$ & $\frac{5}{4} - \frac{x^2}{3}$ & $(x^2-t)$ & $2.05 \times 10^{-7}$ \\
Volterra & $x^{'}$ & $0$ & $0$ & $\cosh(x)+x$ & $e^x + \frac{1}{2}x^2 - 1$ & $-(x-t+1)$ & $4.65 \times 10^{-5}$ \\
Volterra & $x^2+x^{'}$ & $0$ & $0$ & $\sin(x)$ & $\frac{7}{8} + \frac{1}{4}x^2 - \cos(x) + \frac{1}{8}\cos(2x)$ & $-(x-t)$ & $1.11 \times 10^{-4}$ \\
Volterra & $x$ & $2$ & $1$ & $e^x$ & $1+x$ & $(x-t)$ & $2.15 \times 10^{-7}$ \\
Volterra & $x^2$ & $1$ & $1$ & $1+e^{-x}$ & $\frac{9}{4} - \frac{5}{2} x - \frac{1}{2} x^2 - 3 e^{-x} - \frac{1}{4} e^{-2x}$ & $(x-t)$ & $3.48 \times 10^{-6}$ \\
Volterra-Fredholm & $x$ & $1$ & $1$ & $2+6x$ & $9-5x-x^2-x^3$ & $x-t$ & $3.67 \times 10^{-5}$ \\
Volterra-Fredholm & $x$ & $1$ & $1$ & $x e^x$ & $2 e^x - 2$ & $1$ & $6.82 \times 10^{-6}$ \\ \bottomrule
\end{tabular}
\caption{Simulation results for ordinal integro-differential equations using the proposed neural network approach on the interval $\Delta = [0,1]$.}
\label{tbl:oide}
\end{table}

\subsubsection{Partial Integro-differential Equations}
For the next experiment, we examine a two-dimensional unknown function governed by the following partial integro-differential equation:
\begin{equation*}
    \frac{\partial}{\partial t}u(x,t) = \mathcal{S}(x,y) + \int_{\Delta_t} \mathcal{K}(x,t,s) \zeta(u(x,s)) \, ds,
\end{equation*}
where $x \in \Delta_x = [0,1]$ and $t \in \Delta_t = [0,1]$. Table \ref{tbl:pide} presents simulations of various examples of this equation. For each configuration, we determine the source term based on the given exact solution. The initial condition is set as $\hat{u}(x,0) = u(x,0)$ for the exact solution in all scenarios.

\begin{table}[ht]
\begin{tabular}{@{}llllll@{}}
\toprule
Type & $\zeta(x)$ & Exact & Source term & Kernel & MAE \\ \midrule
Fredholm & $x$ & $\sin(x  t)$ & $x  \cos(y  x) + \frac{-1 + \cos(x)}{x}$ & $1$ & $3.07 \times 10^{-4}$ \\
Fredholm & $x$ & $\sin(x  t)$ & $x  \cos(y  x) - x + x  \cos(x)$ & $x^2$ & $5.05 \times 10^{-5}$ \\
Fredholm & $x$ & $\sin(x  t)$ & $x  \cos(y  x) - x  \sin(y) + x  \sin(y)  \cos(x)$ & $x^2 \sin(y)$ & $3.84 \times 10^{-5}$ \\
Fredholm & $x$ & $\sin(x  t)$ & $x  \cos(y  x) + y  \frac{x  \cos(x) - \sin(x)}{x}$ & $x y s$ & $5.03 \times 10^{-5}$ \\
Fredholm & $x^2$ & $\sin(x  t)$ & $x  \cos(y  x) + \frac{\cos(x)  \sin(x) - x}{2x}$ & $1$ & $6.46 \times 10^{-5}$ \\
Volterra & $x$ & $\sin(x  t)$ & $x  \cos(y  x) + \frac{-1 + \cos(x^2)}{x}$ & $1$ & $1.04 \times 10^{-4}$ \\
Volterra & $x$ & $e^{x-t}$ & $-e^{x-t} + 1 - e^{x}$ & $1$ & $6.39 \times 10^{-6}$ \\
\bottomrule
\end{tabular}
\caption{Numerical results from simulating partial integro-differential equations using the proposed architecture.}
\label{tbl:pide}
\end{table}

\subsubsection{Fractional Integro-differential Equations}
For the last experiment in this section, we consider the well-known fractional Volterra model for species population growth in a closed system, as formulated in \cite{afzal2024hyperparameter}:
\begin{equation*}
    \begin{aligned}
        \kappa^{C}\mathfrak{D}^\alpha u(x) &= u(x) - u(x)^2 + u(x) \int_0^x u(t) \, dt, \\
        u(0) &= 0.1,
    \end{aligned}
\end{equation*}
where $\kappa$ is a known non-dimensional parameter, and $_{}^{C}\mathfrak{D}^\alpha$ denotes the Caputo fractional differentiation operator. Using Theorem \ref{thm:matrix}, we have trained a neural network with an appropriate loss function to approximate both the fractional derivatives and the Volterra operator.

We observed that the L-BFGS algorithm does not always converge to the optimal solution. To improve convergence, we first used the Adam optimizer to obtain a good initial weight matrix, and then fine-tuned the weights with the L-BFGS algorithm to achieve more accurate predictions. Since this problem lacks a known analytical solution, we do not report the MAE for this case. However, TeBeest \cite{tebeest1997classroom} has shown that the problem has a unique maximum point, which for $\alpha = 1$ is given by:
\begin{equation*}
    u_{\text{max}} = 1 + \kappa \ln \left(\frac{\kappa}{1 + \kappa - u(0)}\right).
\end{equation*}
In Table \ref{tbl:volterra-population}, we report this maximum value obtained from our proposed method and compare it with TeBeest's results. The simulation outcomes are also visualized in Figure \ref{fig:volterra-population}.

\begin{table}[ht]
\centering
\begin{tabularx}{\textwidth}{@{}cXXXcX@{}}
\toprule
$\kappa$ & $x_{max}$ & $\hat{u}_{max}$ & $u_{max}$ & $x_{max}$ & $u_{max}$ \\ \midrule
 & \multicolumn{3}{c}{$\alpha=1$} & \multicolumn{2}{c}{$\alpha=0.5$} \\ \cmidrule(lr){2-4} \cmidrule(lr){5-6}
0.10 & 0.4745475 & 0.7697309 & 0.7697415 & 0.1505151 & 0.7588339 \\
0.20 & 0.8210821 & 0.6590433 & 0.6590503 & 0.3030303 & 0.6357327 \\
0.30 & 1.1191119 & 0.5840919 & 0.5841117 & 0.4845485 & 0.5483513 \\
0.40 & 1.3846385 & 0.5285197 & 0.5285380 & 0.6625663 & 0.4766903 \\
0.50 & 1.6246625 & 0.4851788 & 0.4851903 & 0.8500850 & 0.4233432 \\
0.60 & 1.8466847 & 0.4502213 & 0.4502255 & 1.0031004 & 0.3807877 \\
0.70 & 2.0507052 & 0.4212958 & 0.4213250 & 1.1596160 & 0.3481358 \\ \bottomrule
\end{tabularx}
\caption{The accuracy criterion for Volterra's population model, considering both ordinal and fractional order derivatives. The values of $u_{\text{max}}$ are obtained following the method described in \cite{tebeest1997classroom}.}
\label{tbl:volterra-population}
\end{table}

\begin{figure}[ht]
    \centering
    \begin{subfigure}[b]{0.47\textwidth}
        \includegraphics[width=\textwidth]{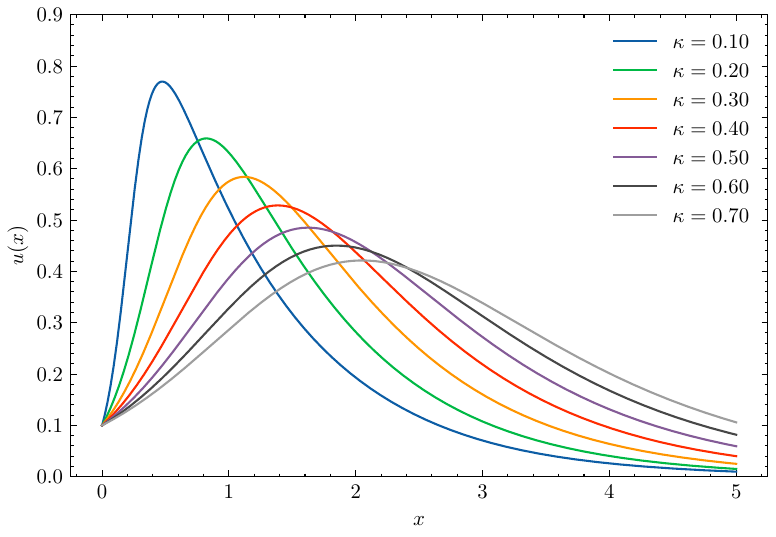}
        \caption{$\alpha=1$}
    \end{subfigure}
    \hfill
    \begin{subfigure}[b]{0.47\textwidth}
        \includegraphics[width=\textwidth]{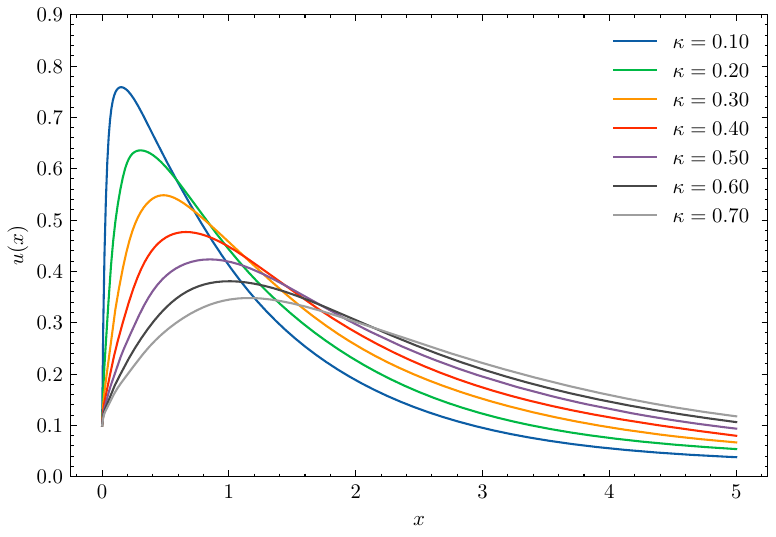}
        \caption{$\alpha=0.5$}
    \end{subfigure}
    \caption{Simulation results of Volterra's population model using the proposed neural network approach for various values of $\kappa$ and different differentiation orders.}
    \label{fig:volterra-population}
\end{figure}

\subsection{Multi-dimensional Integral Equations}
In this section, we evaluate the proposed method for multi-dimensional integral equations. Specifically, we consider a two-dimensional IE of the form:
\begin{equation*}
    \kappa u(x,y) = \mathcal{S}(x,y) + \int_{\Delta_y}\int_{\Delta_x} \mathcal{K}(x,y,s,t) u(s,t) \, dt \, ds,
\end{equation*}
and a three-dimensional IE:
\begin{equation*}
    \kappa u(x,y,z) = \mathcal{S}(x,y,z) + \int_{\Delta_z}\int_{\Delta_y}\int_{\Delta_x} \mathcal{K}(x,y,z,r,s,t) u(r,s,t) \, dt \, ds \, dr.
\end{equation*}
In both cases, the IEs can be classified as either Fredholm or Volterra and as either first-kind or second-kind equations. We apply the formulations presented in Section \ref{sec:3} and simulate various IEs with different kernel functions and exact solutions. Table \ref{tbl:nd} summarizes these examples and presents the simulation results obtained using the proposed neural network method, with the number of training points ranging from 15 to 30 for each dimension.

\begin{table}[ht]
\begin{tabular}{@{}llllll@{}}
\toprule
Type & $\Delta$ & Exact& Source Term& Kernel& MAE \\ \midrule
Fredholm & $[0,1] \times [0,2]$ & $x^2y$ & $x^2y + \frac{4}{9}x$ & $-\frac{1}{2}xt$ & $2.25 \times 10^{-4}$ \\
Fredholm & $[0,1] \times [-1,1] \times [1,2]$ & $x^2y e^x$ & $x^2y e^x - \frac{(e-1)}{2}$ & $e^{sr}$ & $1.11 \times 10^{-3}$ \\
Volterra & $[0,1] \times [0,2]$ & $x+y$ & \makecell{$(x + y - 2) e^{2x + 2y} + (2 - y) e^{x + 2y} +$ \\ $(2 - x) e^{2x + y} + x + y - 2 e^{x + y}$} & $e^{x+y+s+t}$ & $1.39 \times 10^{-3}$ \\
Volterra & $[0,1] \times [0,2]$ & $x+y$ & $x + y + \frac{1}{2} e^{(x + y)} \left(y^2 x + x^2 y\right)$ & $e^{x+y}$ & $3.91 \times 10^{-4}$ \\
Volterra & $[0,1] \times [0,2]$ & $x+y$ & $x + y + \frac{1}{2} e^y \left( y^2 x + x^2 y\right)$ & $e^y$ & $3.62 \times 10^{-4}$ \\
Volterra & $[0,1] \times [0,2]$ & $x+y$ & $x + y + \frac{1}{2} e^x \left(y^2 x + x^2 y\right)$ & $e^x$ & $1.99 \times 10^{-4}$ \\
Volterra & $[0,1] \times [0,2]$ & $x+y$ & $x + y + \frac{1}{2} \left(y^2 x + x^2 y\right)$ & $1$ & $5.78 \times 10^{-5}$ \\ \bottomrule
\end{tabular}
\caption{Numerical solutions of multi-dimensional integral equations using the proposed neural network approach.}
\label{tbl:nd}
\end{table}

\subsection{Systems of Integral Equations}
A system of integral equations consists of multiple interconnected unknown functions appearing within integral terms. For an integer \( M \ge 2 \), such a system can be mathematically defined as:
\begin{equation*}
    \kappa \frac{d^v}{dx^v}u_\iota(x) = \mathcal{S}_\iota(x) + \int_\Delta\sum_{i=1}^\iota \mathcal{K}_{\iota,i}(x,t) u_i(t) \, dt.
\end{equation*}
A combination of the previously discussed types of integral equations can be observed in these systems. For instance, when $\kappa = 0$, the system becomes a set of first-kind integral equations, or it converts into integro-differential equations when $v \in \mathbb{Z}^{>0}$. In this section, we consider the case where $M = 2$, leading to the following system of equations:
\begin{equation}
    \begin{cases}
\kappa u_{1}^{(v)}(x) = \mathcal{S}_1(x) + \displaystyle\int_{\Delta} \mathcal{K}_{1,1}(x,t) u_1(t) + \mathcal{K}_{1,2}(x,t) u_2(t) \, dt, \\
\kappa u_{2}^{(v)}(x) = \mathcal{S}_2(x) + \displaystyle\int_{\Delta} \mathcal{K}_{2,1}(x,t) u_1(t) + \mathcal{K}_{2,2}(x,t) u_2(t) \, dt, \\
\end{cases}
\label{eq:system}
\end{equation}
subject to $v$ boundary conditions when $v > 0$. To simulate these systems, we use two distinct neural networks to approximate the unknown functions, $\mathbf{u}_1 = \mathrm{MLP}_1(\mathbf{X})$ and $\mathbf{u}_2 = \mathrm{MLP}_2(\mathbf{X})$ for $\mathbf{X} \in \mathbb{R}^{N \times 1}$. Each network has its own set of weights $\boldsymbol{\theta}$. Table \ref{tbl:system} presents various configurations of the system \eqref{eq:system} along with the results obtained from the neural network simulations. Once again, we observe acceptable accuracy for the problem, even when dealing with stiff solutions over large domains.

\begin{table}[ht]
\resizebox{\textwidth}{!}{%
\begin{tabular}{@{}cccccccccc@{}}
\toprule
Type & $\Delta$ & $\kappa$ & $v$& $\iota$ & Exact& Source term& $K_{\iota,1}(x,t)$& $K_{\iota,2}(x,t)$& MAE \\ \midrule
\multirow{2}{*}{Fredholm} & \multirow{2}{*}{$[0,\pi]$} & \multirow{2}{*}{1} & \multirow{2}{*}{0} & 1 & $\sin(x) + \cos(x)$ & $\sin(x) + \cos(x) - 4x$ & $x$ & $x$ & $1.30 \times 10^{-3}$ \\
 &  &  &  & 2 & $\sin(x) - \cos(x)$ & $\sin(x) - \cos(x)$ & $1$ & $1$ & $2.74 \times 10^{-4}$ \\
\multirow{2}{*}{Volterra} & \multirow{2}{*}{$[0,1]$} & \multirow{2}{*}{1} & \multirow{2}{*}{0} & 1 & $x$ & $x - \frac{1}{6} x^{4}$ & $(x-t)^{2}$ & $x-t$ & $2.04 \times 10^{-5}$ \\
 &  &  &  & 2 & $x^{2}$ & $x^{2} - \frac{1}{12} x^{5}$ & $(x-t)^{3}$ & $(x-t)^{2}$ & $2.28 \times 10^{-5}$ \\
\multirow{2}{*}{Volterra} & \multirow{2}{*}{$[0,1]$} & \multirow{2}{*}{0} & \multirow{2}{*}{0} & 1 & $1+x$ & $\frac{1}{2} x^{2} + \frac{1}{2} x^{3} + \frac{1}{12} x^{4}$ & $-(x-t-1)$ & $-(x-t+1)$ & $1.18 \times 10^{-3}$ \\
 &  &  &  & 2 & $1+x^{2}$ & $\frac{3}{2} x^{2} - \frac{1}{6} x^{3} + \frac{1}{12} x^{4}$ & $-(x-t+1)$ & $-(x-t-1)$ & $1.17 \times 10^{-3}$ \\
\multirow{2}{*}{Fredholm} & \multirow{2}{*}{$\left[0,\frac{\pi}{2}\right]$} & \multirow{2}{*}{1} & \multirow{2}{*}{2} & 1 & $\cos(x)$ & $-\cos(x) - \left(2 - \frac{\pi}{2}\right)$ & $x-t$ & $x-t$ & $3.69 \times 10^{-5}$ \\
 &  &  &  & 2 & $\sin(x)$ & $-\sin(x) + \left(2 - \frac{\pi}{2}\right)$ & $x+t$ & $x+t$ & $8.72 \times 10^{-5}$ \\
\multirow{2}{*}{Volterra} & \multirow{2}{*}{$[0,1]$} & \multirow{2}{*}{1} & \multirow{2}{*}{1} & 1 & $1 + x + x^{2}$ & $1 + x - \frac{1}{2} x^{2} + \frac{1}{3} x^{3}$ & $x-t$ & $x-t+1$ & $1.26 \times 10^{-6}$ \\
 &  &  &  & 2 & $1 - x - x^{2}$ & $-1 - 3x - \frac{3}{2} x^{2} - \frac{1}{3} x^{3}$ & $x-t+1$ & $x-t$ & $3.39 \times 10^{-6}$ \\ \bottomrule
\end{tabular}
}
\caption{Simulation results of neural network approximations for systems of integral equations.}
\label{tbl:system}
\end{table}

\subsection{Optimal Control Problems}
In this section, we consider an optimal control problem where the objective is to minimize the cost functional \( \mathcal{J} \) defined by:
\begin{equation}
    \min \mathcal{J} = \int_\Delta \mathfrak{L}(\chi(t), u(t), t) \, dt,
    \label{eq:opt-func}
\end{equation}
subject to the dynamics given by:
\begin{equation}
    \mathcal{D}(u, \chi)(t) + \mathcal{I}(u, \chi)(t) = \mathcal{S}(t),
    \label{eq:opt-cons}
\end{equation}
along with specified initial and boundary conditions. Here, \( \chi(t) \in \mathbb{R}^M \) represents the state vector at time \( t \), \( u(t) \in \mathbb{R}^\mathfrak{c} \) denotes the control vector, and \( \mathfrak{L}(\chi(t), u(t), t) \) is the running cost function. The objective is to determine the optimal control \( u^*(t) \) such that \( \mathcal{J}(u^*) = \min_{u \in \mathcal{U}} \mathcal{J}(u) \), where \( \mathcal{U} \) denotes the set of admissible controls.

To tackle this problem, we approximate the state and control functions using neural networks: \( \chi_i(x) = \mathrm{MLP}_i(\mathbf{X}) \) and \( u(x) = \mathrm{MLP}(\mathbf{X}) \). The functional \eqref{eq:opt-func}, being a definite integral over \( \Delta \), can be computed using the Gauss-Legendre matrix-vector product:
\begin{equation*}
    \mathbf{J}_{1 \times 1} = \frac{b-a}{2} \odot \mathbf{L}_{N \times 1}^\top \cdot \mathbf{w}_{N \times 1},
\end{equation*}
where \( \Delta = [a,b] \) and $\mathbf{L}_i=\mathfrak{L}(\chi(\mathbf{x}_i), u(\mathbf{x}_i), \mathbf{x}_i)$. The system dynamics described by Equation \eqref{eq:opt-cons} are typically represented as differential or integro-differential equations, which can be simulated as outlined in previous sections. The loss function is then given by:
\begin{equation*}
    \mathcal{L}(\mathbf{x}) = \mathbf{J} + \gamma \left[ \frac{1}{N \times M} \sum_{i=1}^M \mathfrak{R}_i(\mathbf{x})^\top \mathfrak{R}_i(\mathbf{x}) + \sum_{i=1}^M \sum_{j=0}^{\#B} \{\chi_i^{(j)}(a) - B_{i,j}\}^2 \right],
\end{equation*}
for $M$ different residual functions of $M$ constraints. In this formulation, $\gamma$ is the regularization or trade-off parameter that balances the influence of the optimal control objective against the constraints. We will provide a comprehensive discussion of this parameter in Example \ref{ex:opt-delay}. For now, we begin by testing several optimal control problems using the proposed approach.

\begin{example}
\label{ex:opt-ord-1}
Consider the following optimal control problem:
\begin{equation*}
\begin{aligned}
\min &\quad \mathcal{J} = \int_{0}^{1} \left( u(t)^2 + \chi(t)^2 \right) dt, \\
\text{s.t.} &\quad u(t) = \frac{d}{dt} \chi(t),\\
&\quad \chi(0) = 0, \quad \chi(1) = \frac{1}{2},
\end{aligned}
\end{equation*}
with the exact solution \cite{kafash2012application}:
\begin{equation*}
\begin{aligned}
    u(t) &= \frac{e \left( e^t + e^{-t} \right)}{2e^2 - 2},\\
    \chi(t) &= \frac{e \left( e^t - e^{-t} \right)}{2e^2 - 2},
\end{aligned}
\end{equation*}
which yields the optimal cost \( \mathcal{J} = 0.328259 \). Using the proposed approach, we conducted simulations with $100$ training points derived from the roots of Legendre polynomials. The setup included two MLP networks with an architecture of $[1,10,10,1]$ and hyperbolic tangent as the activation function, using $\gamma = 10^3$. The results of this simulation are presented in Table \ref{tbl:opt-cont}.

\end{example}

\begin{example}
\label{ex:opt-ord-2}
Consider this optimal control problem \cite{kafash2012application}:
\begin{equation*}
\begin{aligned}
\min &\quad \mathcal{J} = \frac{1}{2} \int_{0}^{1} \left[ u(t)^2 + \chi(t)^2 \right] \, dt \\
\text{s.t.} &\quad u(t) = \frac{d}{dt} \chi(t) + t,
\end{aligned}
\end{equation*}
and the initial condition \( \chi(0) = 1 \). The exact solution to this problem, which gives \( \mathcal{J} = 0.192909 \), is given by \( \chi(t) = \kappa \exp(\sqrt{2} t) + (1 - \kappa) \exp(-\sqrt{2} t) \) and \( u(t) = \kappa (\sqrt{2} + 1) \exp(\sqrt{2} t) - (1 - \kappa) (\sqrt{2} - 1) \exp(-\sqrt{2} t) \), where
\begin{equation*}
    \kappa := \frac{2\sqrt{2} - 3}{-e^{2\sqrt{2}} + 2\sqrt{2} - 3}.
\end{equation*}
Table \ref{tbl:opt-cont} presents the results of the neural network simulation for this problem, utilizing the same hyperparameters as those used in Example \ref{ex:opt-ord-1}.

\end{example}

\begin{example}
\label{ex:opt-frac}
In this example, we evaluate the performance of the proposed neural network model on an optimal control problem subject to fractional constraints:
\begin{equation*}
\begin{aligned}
\min & \quad \mathcal{J} = \int_{0}^{1} \left[ \left(\chi_1(t) - 1 - t^\frac{3}{2}\right)^2 + \left(\chi_2(t) - t^\frac{5}{2}\right)^2 + \left(u(t) - \frac{3\sqrt{\pi}}{4}t + t^\frac{5}{2}\right)^2 \right] dt, \\
\text{s.t.} &\quad \frac{d^\frac12}{dt^\frac12} \chi_1(t) = \chi_2(t) + u(t), \\
 &\quad \frac{d^\frac12}{dt^\frac12} \chi_2(t) = \chi_1(t) + \frac{15\sqrt{\pi}}{16} t^2 - t^\frac{3}{2} - 1,
\end{aligned}
\end{equation*}
with initial conditions \( \chi_1(0) = 1 \) and \( \chi_2(0) = 0 \). The exact solution to this problem is given by \( \chi_1(t) = 1 + t^{\frac{3}{2}} \), \( \chi_2(t) = t^{\frac{5}{2}} \), and \( u(t) = \frac{3\sqrt{\pi}}{4} t - t^{\frac{5}{2}} \). The analytical solution to this problem yields \( \mathcal{J} = 0 \) in the optimal case. Table \ref{tbl:opt-cont} reports the results of the neural network simulation of this problem in which $N=2000$ and $\gamma=10$.

\end{example}

\begin{example}
    \label{ex:opt-delay}
    For the next experiment, we consider an optimal control problem with a delay differential equation constraint \cite{rabiei2023hybrid}:
\begin{equation*}
\begin{aligned}
\min &\quad \mathcal{J} = \frac{1}{2} \int_{0}^{2} \left( u(t)^2 + \chi(t)^2 \right) dt, \\
\text{s.t.} &\quad \frac{d}{dt} \chi(t) = u(t) + \chi(t-1), \quad &t \in [0, 2],\\
&\quad \chi(t) = 1, \quad &t \in [-1, 0].
\end{aligned}
\end{equation*}
This problem presents two complexities. First, it has no known exact solution, so we cannot directly evaluate our results. Second, it is sensitive to the parameter \( \gamma \). Unlike previous examples, in this case, it is challenging to find a reasonably smooth function for \( u(t) \). To illustrate how the parameter \( \gamma \) affects both the optimal value of the problem and the residual value of the constraints, we have plotted the relationship between these two in Figure \ref{fig:opt-gamma}. It can be observed that as \( \gamma \) increases, the constraints become more accurate, while the minimum value of \( \mathcal{J} \) increases. The intersection of these two graphs can be seen as an optimal value of \( \gamma \) that balances the minimization of both the objective functional and the constraints' residuals. Therefore, we select \( \gamma = 750 \) with \( 200 \) training points and present the simulation results for this problem in Table \ref{tbl:opt-cont}.

\begin{figure}[ht]
    \centering
    \includegraphics[width=8cm]{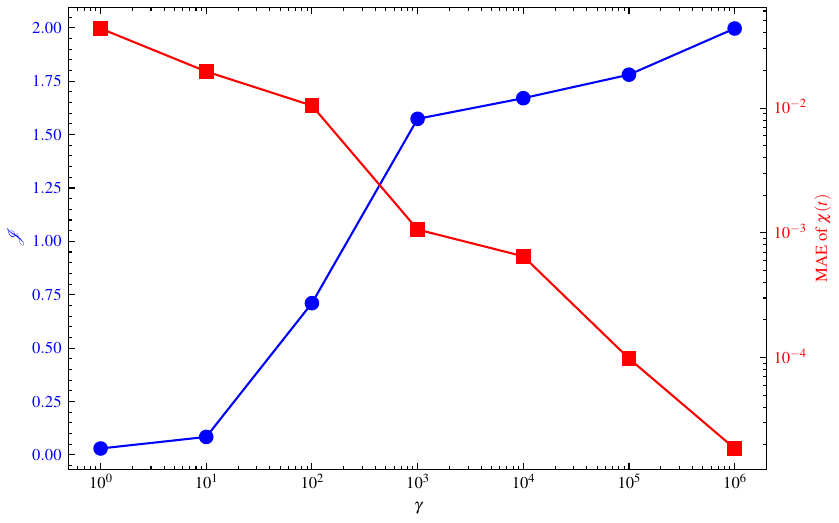}
    \caption{The trade-off between the residual of the state variable and the system's cost for Example \ref{ex:opt-delay}. Increasing the value of $\gamma$ shifts the focus more towards satisfying the constraints while reducing the emphasis on the control variable. This leads to a decrease in the MAE of $\chi(t)$ but results in a higher overall cost. The intersection point of the two line graphs can be considered an optimal value for $\gamma$.}
    \label{fig:opt-gamma}
\end{figure}
\end{example}

\begin{example}
\label{ex:opt-nonlinear}
For the next example, we consider a nonlinear optimal control problem of the form \cite{rabiei2023approach}:
\begin{equation*}
\begin{aligned}
\min \quad & \mathcal{J} = \frac{1}{2}\int_{0}^{1} \left[ u(t)^2 + \frac{5}{4}\chi(t)^2 + \chi(t)u(t) \right] dt \\
\text{s.t.} &\quad  \frac{d}{dt} \chi(t) = \frac{1}{2}\chi(t) + u(t), \\
&\quad \chi(0) = 1.
\end{aligned}
\end{equation*}
The exact solution to this problem is given by \( u(t) = {-(\tanh(1 - t) + 0.5) \cosh(1 - t)}/{\cosh(1)} \) and \( \chi(t) = {\cosh(1 - t)}/{\cosh(1)} \). The optimal value for this problem is \( \mathcal{J} = 0.380797077 \). Table \ref{tbl:opt-cont} reports the results of the neural network simulation for this problem using $\gamma=10^4$ and $N=500$.

\end{example}

\begin{example}
    \label{ex:opt-ide}
In this example, we consider an optimal control problem with an integro-differential equation constraint \cite{heydari2023orthonormal}:
\begin{equation*}
        \begin{aligned}
            \min \quad & \mathcal{J} = \int_0^1 \left[\chi(t) - e^{t^2}\right]^2 + \left[u(t) - \left(2t + 1\right)\right]^2 \, dt,\\
            \text{s.t.} &\quad  \frac{d}{dt} \chi(t) = u(t) - \chi(t) + \int_0^t \left(t(2t + 1) e^{s(t-s)} \chi(s)\right) ds,\\
            &\quad \chi(0) = 1.
        \end{aligned}
    \end{equation*}
The exact solution to this problem is given by \(\chi(t) = \exp(t^2)\) and \(u(t) = 2t + 1\), which yields the optimal value \(\mathcal{J} = 0\). To formulate the constraints, we use the Volterra operator approximation technique outlined in Section \ref{sec:3} and train the network with 100 training points using \( \gamma = 10^3 \). The results of this neural network simulation are reported in Table \ref{tbl:opt-cont}.

\end{example}

\begin{example}
    \label{ex:opt-2d}
For a more complex optimal control problem, consider the following two-dimensional problem \cite{hassani2020generalized}:
\begin{equation}
        \begin{aligned}
            \min \quad & \mathcal{J} =  \int_0^1 \int_0^1 \left[\left(\chi(s,t) - t^4 \sin(s)\right)^2 + \left(u(s,t) - t^3 \cos(s)\right)^2 \right] ds \, dt, \\
            \text{s.t.} &\quad  \frac{\partial \chi(s, t)}{\partial t} = \cos(\chi(s, t)) + 2 \sin(s) \frac{\partial \chi(s, t)}{\partial s} + \frac{\partial^2 \chi(s, t)}{\partial s^2} \\
            &\quad \quad \quad \quad + 6 \sin(s) u(s, t) - \cos(t^4 \sin(s)) - t^3 \left(t \sin(2s) - t \sin(s) + 3 \sin(2s)\right) \\
            &\quad \quad \quad \quad + 4 \sin(s) t^3, \\
            &\quad \chi(s,0) = 0, \\
            &\quad \chi(0,t) = 0.
        \end{aligned}
    \end{equation}
The exact solution to this problem is \( \chi(s,t) = t^4 \sin(s) \) and \( u(s,t) = t^3 \cos(s) \), yielding the analytical objective value \( \mathcal{J} = 0 \). Table \ref{tbl:opt-cont} presents the results of the neural network simulation for this problem. We used a nested Gauss-Legendre method to approximate the cost functional and automatic differentiation to compute the partial derivatives of the unknown solution with respect to \( t \). In this case, \( N = 25 \) training points were employed for each dimension.

\end{example}

\begin{table}[ht]
\centering
\begin{tabularx}{\textwidth}{@{}XXXXXX@{}}
\toprule
Example & Type & Best $\mathcal{J}$ & Simulated $\mathcal{J}$ & Function & MAE \\ \midrule
\multirow{2}{*}{Example \ref{ex:opt-ord-1}} & \multirow{2}{*}{Ordinal} & \multirow{2}{*}{$0.328259$}& \multirow{2}{*}{$0.326641$}& $u$ & $5.35 \times 10^{-3}$ \\
 &  &  &  & $x_1$ & $7.12 \times 10^{-4}$ \\
\multirow{2}{*}{Example \ref{ex:opt-ord-2}} & \multirow{2}{*}{Ordinal} & \multirow{2}{*}{$0.192909$}& \multirow{2}{*}{$0.192904$}& $u$ & $4.59 \times 10^{-3}$ \\
 &  &  &  & $x_1$ & $6.22 \times 10^{-4}$ \\
\multirow{3}{*}{Example \ref{ex:opt-frac}} & \multirow{3}{*}{Fractional} & \multirow{3}{*}{$0.0$}& \multirow{3}{*}{$0.000001$} & $u$ & $8.63 \times 10^{-4}$ \\
 &  &  &  & $x_1$ & $2.46 \times 10^{-4}$ \\
 &  &  &  & $x_2$ & $1.13 \times 10^{-4}$ \\ 
 \multirow{2}{*}{Example \ref{ex:opt-delay}} & \multirow{2}{*}{Delay} & \multirow{2}{*}{$1.647874^{*}$} & \multirow{2}{*}{$1.548422$}& $u$ & - \\
 &  &  &  & $x_1^*$ & $1.45 \times 10^{-3}$ \\
\multirow{2}{*}{Example \ref{ex:opt-nonlinear}} & \multirow{2}{*}{Non-linear} & \multirow{2}{*}{$0.380797$}& \multirow{2}{*}{$0.380797$}& $u$ & $7.43 \times 10^{-3}$ \\
 &  &  &  & $x_1$ & $7.53 \times 10^{-4}$ \\
 \multirow{2}{*}{Example \ref{ex:opt-ide}} & \multirow{2}{*}{Integro-differential} & \multirow{2}{*}{$0.0$}& \multirow{2}{*}{$0.000049$}& $u$ & $6.14 \times 10^{-3}$ \\
 &  &  &  & $x_1$ & $6.58 \times 10^{-4}$ \\
  \multirow{2}{*}{Example \ref{ex:opt-2d}} & \multirow{2}{*}{2D} & \multirow{2}{*}{$0.0$}& \multirow{2}{*}{$0.000016$}& $u$ & $3.94 \times 10^{-3}$ \\
 &  &  &  & $x_1$ & $1.22 \times 10^{-4}$ \\
 \bottomrule
\end{tabularx}
\caption{Simulation results of the proposed neural network model for various optimal control problems are presented. (*) The mean absolute error (MAE) of the residual function is reported for the delay optimal control problem, as no exact solution is available. The value of \(\mathcal{J}\) for this case has not been computed analytically; instead, the reported value is derived from numerical simulations as provided in \cite{rabiei2023hybrid}.}
\label{tbl:opt-cont}
\end{table}

\subsection{Inverse Problems}
In previous sections, we explored the application of neural networks to forward integral equations and optimal control problems. This section focuses on the inverse problem, where we aim to infer the underlying dynamics and governing equations of complex systems from observed data.

To demonstrate how the proposed method can be applied to the inverse form of mathematical equations involving integral operators, we consider the following integral equations with the unknown parameter \(\kappa\in \mathbb{R}\):
\begin{align}
    u(x) &= \mathcal{S}(x) + \kappa \int_0^1 \exp(2 t) \, u(t) \, dt \tag{EX.5}, \\
    u(x) &= \mathcal{S}(x) + \kappa \int_0^x t^2 \, u(t) \, dt \tag{EX.6},
\end{align}
where \(\mathcal{S}(x)\) is a known function. The exact solutions to these problems are \(x^3 + x\) and \(\cos(x)\), respectively. To simulate this problem and determine the unknown \(\kappa\) as well as \(u(x)\), we assume that the system dynamics are given at \(x = [0, 0.25, 0.5, 0.75, 1]\). We then construct the physics-informed loss function as described in \eqref{eq:loss} and train the network until convergence. We simulated these problems with five different values of \(\kappa\). In all experiments, we observed satisfactory convergence in both the system dynamics and the unknown parameter. Figure \ref{fig:data-driven} presents the simulation results.

\begin{figure}[ht]
    \centering
    \includegraphics[width=1\textwidth]{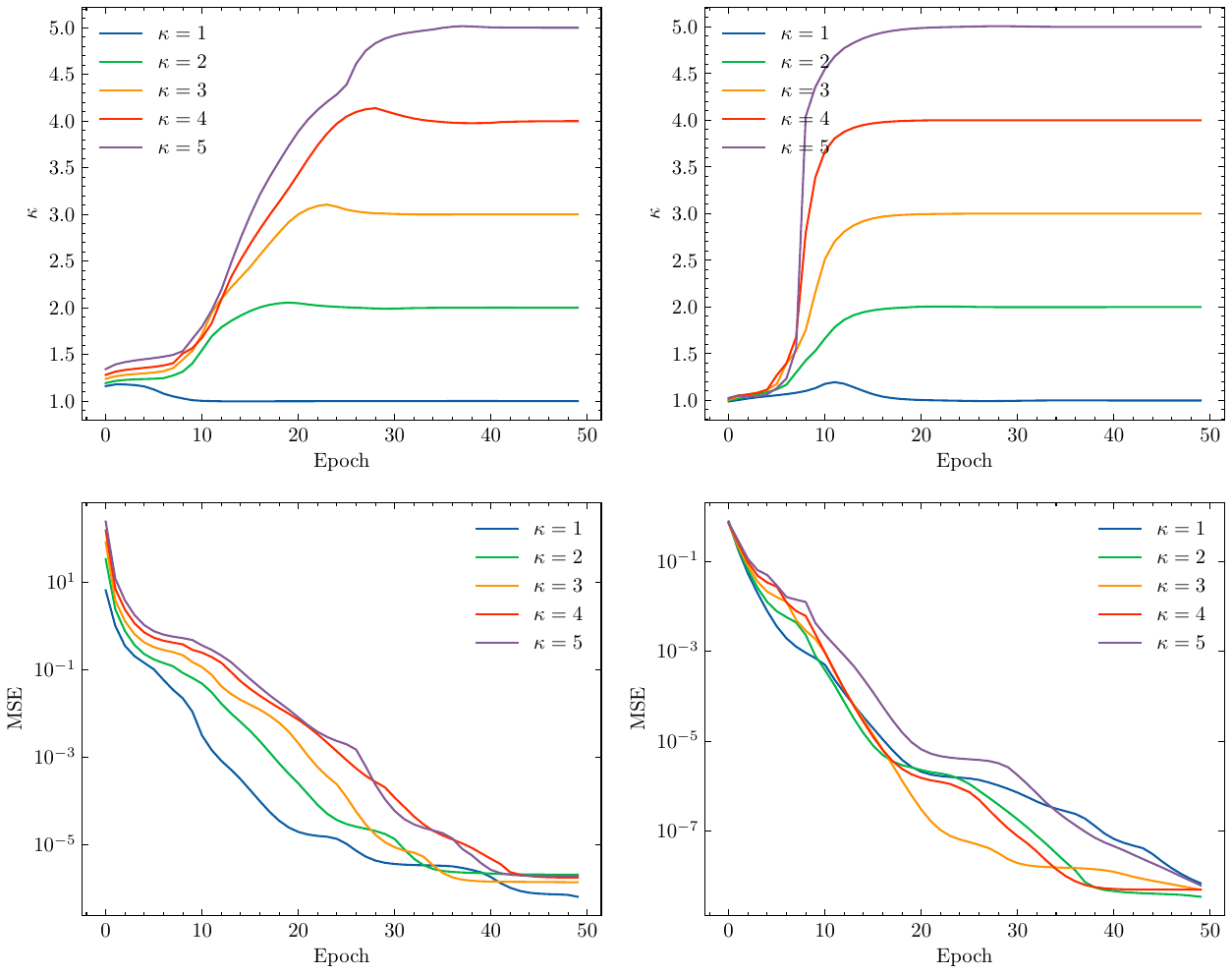}
    \caption{Data-driven solutions for problems \textbf{(EX.5)} (left) and \textbf{(EX.6)} (right). The network may require different numbers of epochs to converge to the exact value for each \( \kappa \). In both cases, the network accurately identifies the unknown function \( u(x) \).}
    \label{fig:data-driven}
\end{figure}

To tackle a more complex problem, we consider the following non-linear Volterra fractional integro-differential equation:
\begin{equation*}
    ^{C}\mathfrak{D}^\frac{1}{2} u(x) = \mathcal{S}(x) + \kappa(x) \int_0^x t \left[\frac{d}{dx}u(x)\right]^2 \, dt,
\end{equation*}
where the exact solution is given by:
\begin{equation*}
    u(x) = -3 + 2x - \frac{4(x - 2)^3}{3} + \frac{4(x - 2)^5}{15} - \frac{8(x - 2)^7}{315} + \frac{(x - 2)^9}{945}.
\end{equation*}
In this problem, we assume that both $\kappa(x)$ and the initial condition are unknown. To investigate the simulation of this inverse problem, we generate a set of $50$ random data points from the exact solution over the domain $\Delta = [0,4]$, with added white noise at a fraction of $0.08$. We then approximate the solution using a neural network with an architecture of $[1,10,10,1]$, trained on 100 collocation points. The loss function for this task is defined as:
\begin{equation*}
    \mathcal{L}(\mathbf{X}) = \frac{1}{N} \mathfrak{R}(\mathbf{X})^\top \mathfrak{R}(\mathbf{X}) + \lambda^{\text{Data}} \mathrm{MSE}^{\text{Data}},
\end{equation*}
where $\lambda^{\text{Data}} = 1$, and $\mathfrak{R}(\mathbf{X})$ is given by:
\begin{equation*}
    \mathfrak{R}(\mathbf{X}) = \mathcal{M} \cdot \mathbf{u} - \boldsymbol{\kappa} \odot \mathbf{I} - \mathbf{S},
\end{equation*}
with $\mathbf{u} = \mathrm{MLP}(\mathbf{X})$, $\boldsymbol{\kappa}$ as a trainable vector, $\mathbf{I}$ representing the quadrature, and $\mathbf{S} = \mathcal{S}(\mathbf{X})$. 

After training the network, the simulation results are shown in Figure \ref{fig:inverse}. It is evident that the proposed method effectively identifies the unknown dynamics of the problem, even when the unknown parameters are complex.

\begin{figure}[ht]
    \centering
    \includegraphics[width=0.9\textwidth]{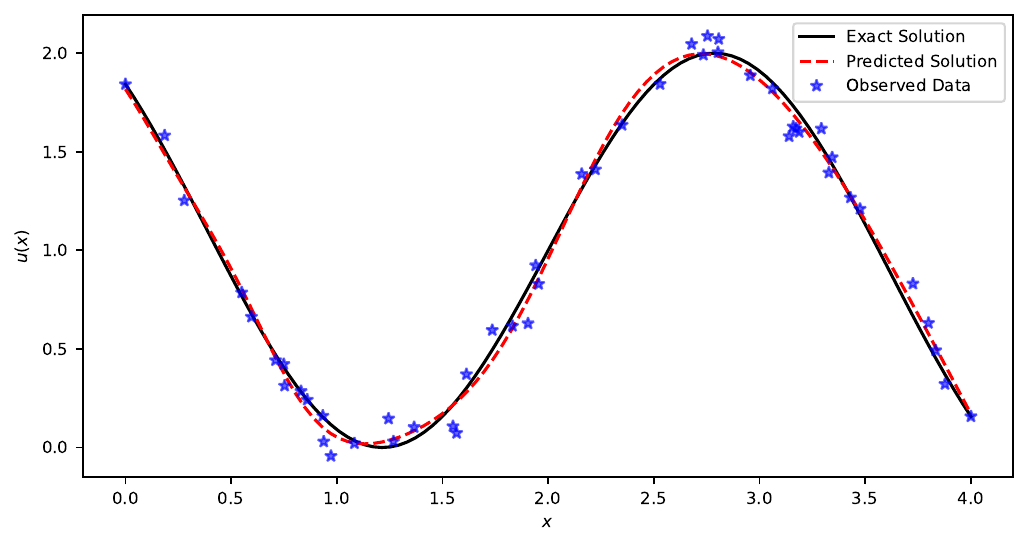}
    \caption{The simulation result for an inverse fractional Volterra integro-differential equation.}
    \label{fig:inverse}
\end{figure}
\section{Python Package}
To enhance the usability of the proposed method, we developed a Python package called \texttt{pinnies} for simulating problems involving integral operators with neural networks \footnote{\url{https://github.com/alirezaafzalaghaei/pinnies}}. You can install this package using the command \texttt{pip install pinnies}. 

For demonstration, we use a one-dimensional Volterra integro-differential equation of the second kind:
\begin{equation}
    \begin{aligned}
    u'(x) &+ u(x) = \int_0^x e^{t-x} \zeta(u(t)) \, dt, \\
    u(0) &= 1, 
\end{aligned}
\label{eq:library}
\end{equation}
where \(\zeta(x) = x\), and the exact solution is \(u(x) = \exp(x) \cosh(x)\). To implement this problem using the `pinnies` package, you should define a class as follows:

\begin{minted}{python}
from pinnies import Volterra
class TestProblem(Volterra):
    def __init__(self, domain, num_train, model):
        # Initialize the TestProblem instance with the domain, number of training points, and model
        super().__init__(domain, num_train)  # Call the parent class constructor
        self.a, self.b = domain  # Set the domain boundaries
        self.model = model  # Store the machine learning model
        self.K = torch.exp(self.T - self.X)  # Define the kernel for the Volterra operator

    def residual(self):
        y = self.predict(self.x)  # Get model prediction at self.x
        y_x = self.diff(y, self.x, n=1)  # Compute the first derivative of the prediction

        u_t = self.predict_u_t()  # Predict the function u for at self.T
        zeta = u_t  # apply possible non-linearity
        
        I = self.quad(self.K * zeta, self.a, self.x)  # Compute the integral from a to self.x

        initial = self.get_initial()  # Get the initial condition residual
        
        return [y_x + y - I, initial]  # Return the final residual

    def get_initial(self):
        zero = torch.tensor([[0.0]])  # Define a tensor for the initial condition at x = 0
        
        # Compute the difference between the model prediction at x = 0 and the expected initial value (1)
        return self.predict(zero) - 1

    def exact(self, x):
        # The exact solution for the problem, used for validation purposes
        return torch.exp(-x) * torch.cosh(x)
\end{minted}

Next, define a neural network architecture and use it to create an instance of the \texttt{TestProblem} class. Solve the problem using the \texttt{solve} method:
\begin{minted}{python}
from torch import nn

model = nn.Sequential(  # Define a simple MLP model
    nn.Linear(1, 10),
    nn.Tanh(),
    nn.Linear(10, 10),
    nn.Tanh(),
    nn.Linear(10, 1),
)

p = TestProblem([0, 5], num_train = 10, model)  # Instantiate the problem with domain, points, and model
p.solve(epochs=30, learning_rate=0.1)  # Solve the problem with specified iterations and learning rate

y_test = p.predict(x_test)  # Test the model with some input data
\end{minted}

The simulation results from this code snippet are shown in Figure \ref{fig:compare-libs}. This figure compares the results with those obtained using the well-known DeepXDE Python package \cite{lu2021deepxde}, which implements the \texttt{IDE} class for solving Volterra-type integro-differential equations. The simulation configuration includes a quadrature degree of $10$, with $10$ training points over the domain $[0,5]$ and a similar MLP architecture of $[1,10,10,1]$. For training the \texttt{IDE}, we use $100$ epochs with a learning rate of $0.1$. As shown, the proposed package is both faster and more accurate, which is advantageous for precise simulations. Additionally, the \texttt{pinnies} package offers other types of operators, such as Fredholm, Volterra-Fredholm, and multi-dimensional operators, which are not available in the \texttt{deepxde} package.

In addition to the DeepXDE package, there is another PINN package called \texttt{NeuralPDE.jl}, implemented in the Julia programming language for physics-informed neural network tasks \cite{zubov2021neuralpde}. However, \texttt{NeuralPDE.jl} has limited support for integro-differential equations, specifically only handling one-dimensional Volterra operators with a fixed kernel function, $\mathcal{K}(x,t)=1$. This limitation makes it incompatible with the requirements of our tests, so we did not include this method in our analysis. Additionally, our experiments showed that \texttt{NeuralPDE.jl} has significantly slower simulation times compared to \texttt{deepxde} and \texttt{pinnies}, further justifying its exclusion from our study.

\begin{figure}[ht]
    \centering
    \includegraphics[width=1\textwidth]{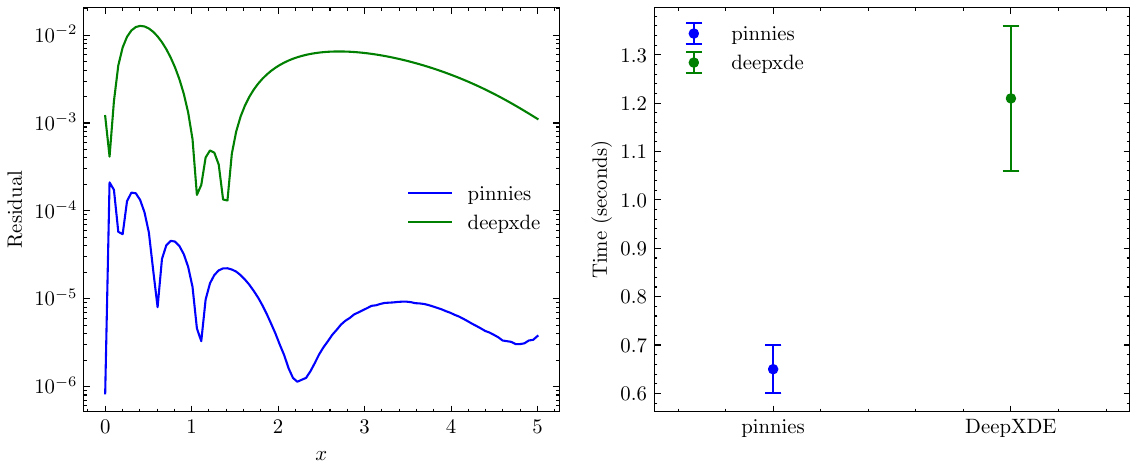}
    \caption{A comparison between the \texttt{pinnies} and \texttt{deepxde} packages for solving a one-dimensional Volterra integro-differential equation, as defined in Equation \eqref{eq:library}, is presented. This comparison shows that the \texttt{pinnies} package not only achieves higher accuracy but also performs faster than \texttt{deepxde}.}
    \label{fig:compare-libs}
\end{figure}

\section{Concluding Remarks}
In this paper, we propose an efficient tensor-vector product approach for the fast and accurate approximation of integral operators commonly found in the mathematical modeling of natural and engineering phenomena. We applied this method to a variety of problems involving integral operators, including integral equations, integro-differential equations, and optimal control problems, potentially involving fractional derivatives. To achieve this, we utilized Gaussian quadrature formulas to approximate the integral operators and automatic differentiation to compute analytical derivatives of integer order.

The simulation results demonstrate that the proposed method is robust against variations in deep learning model hyperparameters, such as network architecture and training data. Additionally, the integration approach is stable and notably fast, outperforming even PyTorch's built-in automatic differentiation ($8.3\pm0.03$ vs. $21\pm0.89$ microseconds). In most cases, the proposed method achieves near-floating-point accuracy within deep learning frameworks, ensuring near-analytical precision in solving the given problems.

We tested our approach for approximating integral terms on various types of integral equations, including Fredholm and Volterra equations of the first and second kinds, in both one-dimensional and multi-dimensional cases, as well as systems of these equations. We also applied it to integro-differential equations of ordinary, partial, and fractional orders. Notably, we demonstrated how a finite difference discretization technique can approximate the fractional Caputo derivative, successfully solving Volterra's population model with acceptable accuracy. Furthermore, we extended the model to address various optimal control problems, including multi-dimensional, fractional, delayed, integro-differential, and nonlinear problems, showing that the model can achieve high accuracy in these cases. In addition to these forward problems, we demonstrated how the proposed method could solve inverse forms of fractional integro-differential equations and remain robust even with noisy data.

Despite its versatility and efficiency, the method has certain limitations. The most significant challenge lies in implementing the proposed method in non-rectangular problem domains due to the inherent constraints of Gaussian quadrature. This issue could be addressed in future work by exploring alternative numerical quadrature methods, such as Monte Carlo approximations. Another promising direction for future research is the application of adaptive integral approximation techniques, which could be particularly beneficial for handling stiff problems. Clenshaw–Curtis quadrature, based on the fast Fourier transform, is another accurate and adaptive numerical integration approach that may offer a viable alternative.

\backmatter

\section*{Conflict of interests}
The authors declare that there are no conflicts of interest.

\section*{Funding}
None.

\section*{Data Availability}
The authors confirm that the data supporting the findings of this study are available within the article.

\section*{CRediT Authorship}
\noindent
\textbf{Alireza Afzal Aghaei}: Formal Analysis, Methodology, Software, Writing – review \& editing\\
\textbf{Mahdi Movahedian Moghaddam}: Writing – original draft, Resources, Software\\
\textbf{Kourosh Parand}: Resources\\

\begingroup
\setstretch{1.0}
\bibliography{sn-bibliography}
\endgroup

\end{document}